\documentclass{article}

\PassOptionsToPackage{numbers, sort, compress}{natbib}
\usepackage[preprint]{nips_2018}

\usepackage[utf8]{inputenc} % allow utf-8 input
\usepackage[T1]{fontenc}    % use 8-bit T1 fonts
\usepackage{hyperref}       % hyperlinks
\usepackage{url}            % simple URL typesetting
\usepackage{booktabs}       % professional-quality tables
\usepackage{amsfonts}       % blackboard math symbols
\usepackage{nicefrac}       % compact symbols for 1/2, etc.
\usepackage{microtype}      % microtypography

\usepackage{calc}
\usepackage{amsmath, amssymb, amsthm}
\usepackage{parskip}
\usepackage{color,hyperref}
\usepackage{epsfig}
\usepackage{subfig}
\usepackage{verbatim}
\usepackage{scalerel}
\usepackage{rotating}
\usepackage[ruled,noend]{algorithm2e}
\usepackage{algorithmic}
\usepackage{enumitem}
\usepackage{wrapfig}
\usepackage{cleveref}
\usepackage{bbm}
\usepackage{array}
\usepackage{soul}
\usepackage{xspace}
\usepackage{kky}
\usepackage{coxDefns}

\newcommand{\akshay}[1]{}

\newenvironment{condenum}
 {\begin{enumerate}[ref=\thecondition.\arabic*]}
 {\end{enumerate}}

\title{\Large Myopic Bayesian Design of Experiments via \\
Posterior Sampling and Probabilistic Programming}

\author{
Kirthevasan Kandasamy \\
Carnegie Mellon University \\
\incmtt{kandasamy@cs.cmu.edu}
\And
Willie Neiswanger \\
Carnegie Mellon University \\
\incmtt{willie@cs.cmu.edu}
\And
Reed Zhang \\
Carnegie Mellon University \\
\incmtt{rrz@andrew.cmu.edu}
\And
Akshay Krishnamurthy \\
U-Mass, Amherst \\
\incmtt{akshay@cs.umass.edu}
\And
Jeff Schneider \\
Carnegie Mellon University \\
\incmtt{schneide@cs.cmu.edu}
\And
Barnab\'as P\'oczos \\
Carnegie Mellon University \\
\incmtt{bapoczos@cs.cmu.edu}
}

\begin{document}

\maketitle
% Tables for Cox project

\newcommand{\insertAlgoMain}{
\begin{algorithm}
\begin{algorithmic}[1]
\REQUIRE Prior $\pthetazero$ for $\thetatrue$, Conditional distribution
  $\PP(Y|X,\theta)$.
\STATE $\datatt{0}\leftarrow \emptyset$.
\label{step:initdata}
\FOR{$t=1,2,\dots$}
\STATE Sample $\theta\sim\pthetatt{t-1} \equiv \PP(\thetatrue|\datatmo)$.
\label{step:sampletheta}
\STATE Choose $\Xt = \argmin_{x\in\Xcal} \penlplustmo(\theta, \datatmo, x)$.
\label{step:choosext}
\STATE $\YXt\leftarrow$ conduct experiment at $\Xt$.
\label{step:experiment}
\STATE Set $\datat \leftarrow \datatmo \cup \{(\Xt, \YXt)\}$.
\label{step:adddata}
\ENDFOR
\caption{$\;$\bdoes ($\pspolicy$) \label{alg:bdoe}}
\end{algorithmic}
\end{algorithm}
}

% figures for Cox

\newcommand{\insertFigExpone}{ 
\newcommand{\expfigwidth}{1.41in} 
\newcommand{\expfighsp}{\hspace{-0.04in}}
\newcommand{\insertleftspace}{\hspace{-0.3in}}
\begin{figure*}
\begin{center}
% \centering
\insertleftspace
\subfloat{
\includegraphics[width=\expfigwidth]{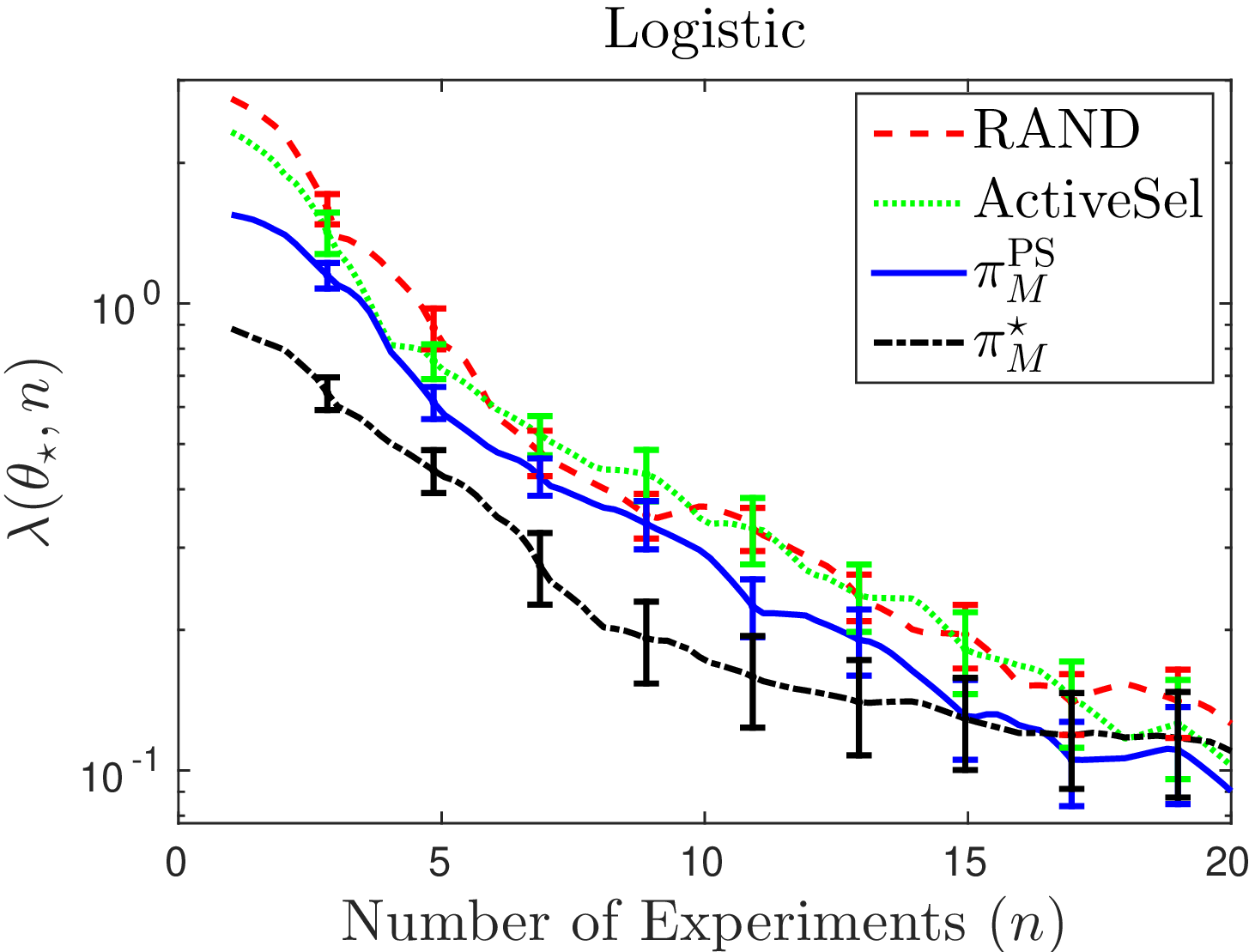}
\label{fig:expeg1}}\expfighsp
\subfloat{
\includegraphics[width=\expfigwidth]{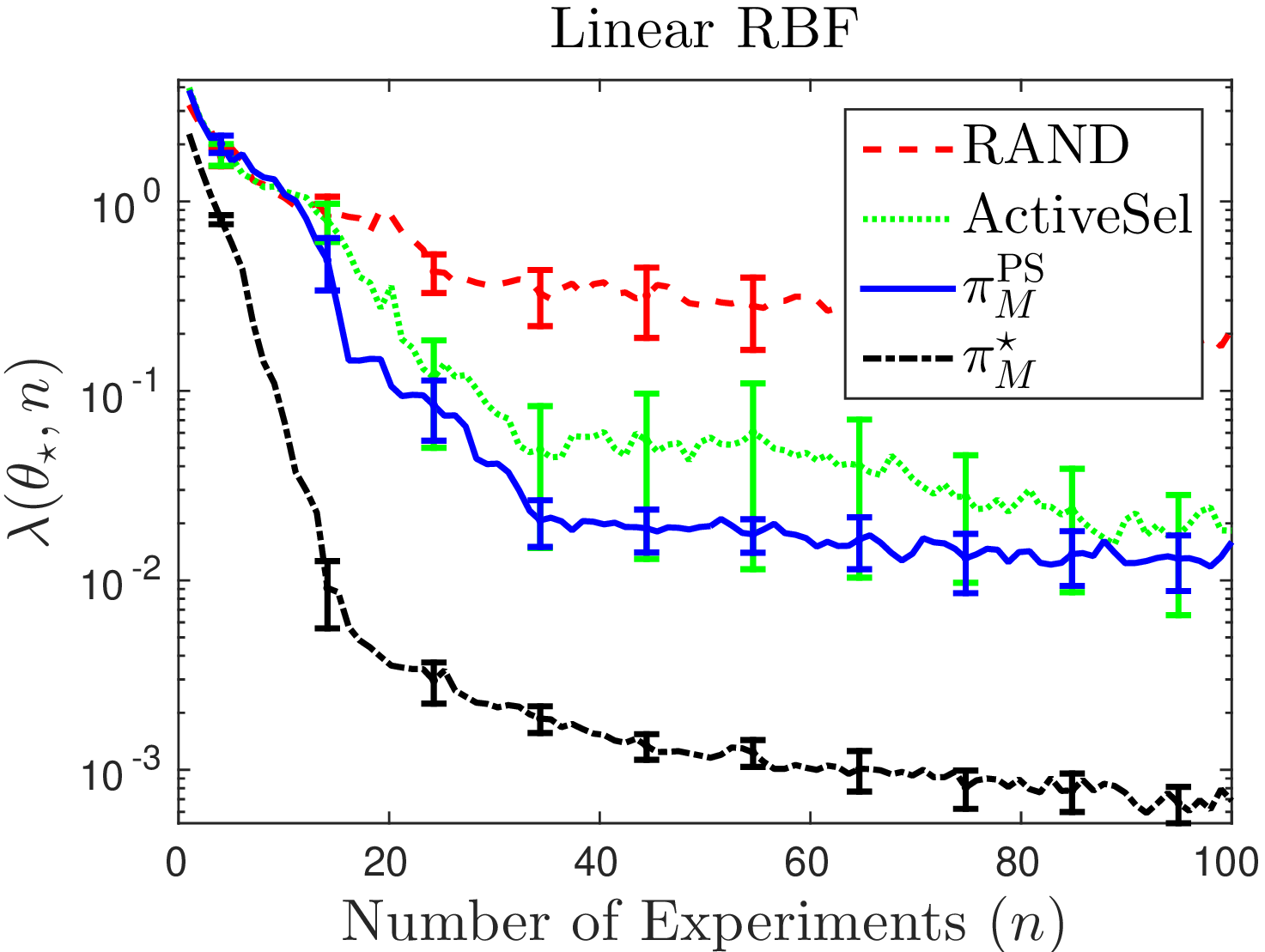}
\label{fig:expeg2}}\expfighsp
\subfloat{
\includegraphics[width=\expfigwidth]{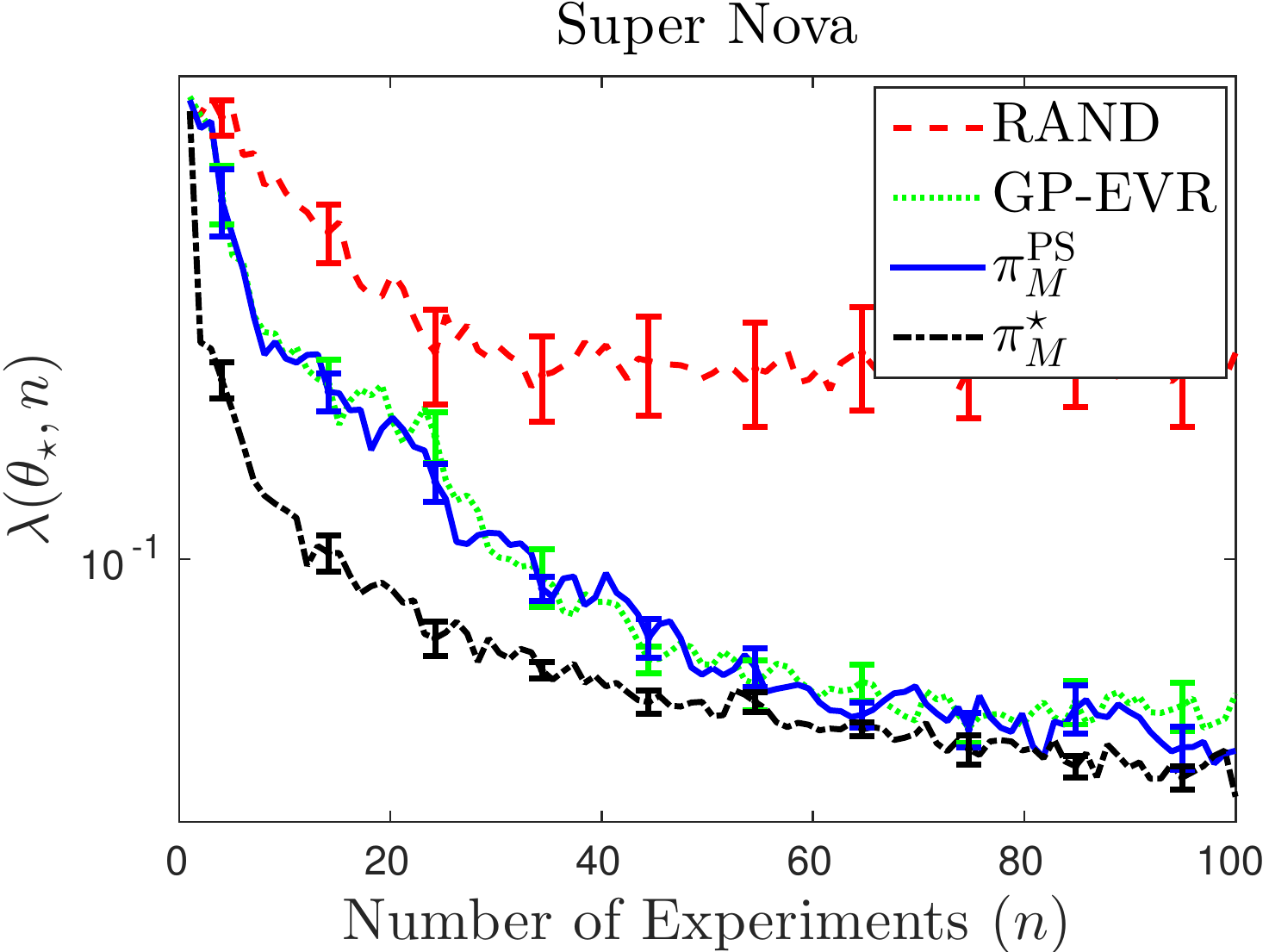}
\label{fig:expeg3}}\expfighsp
\subfloat{
\includegraphics[width=\expfigwidth]{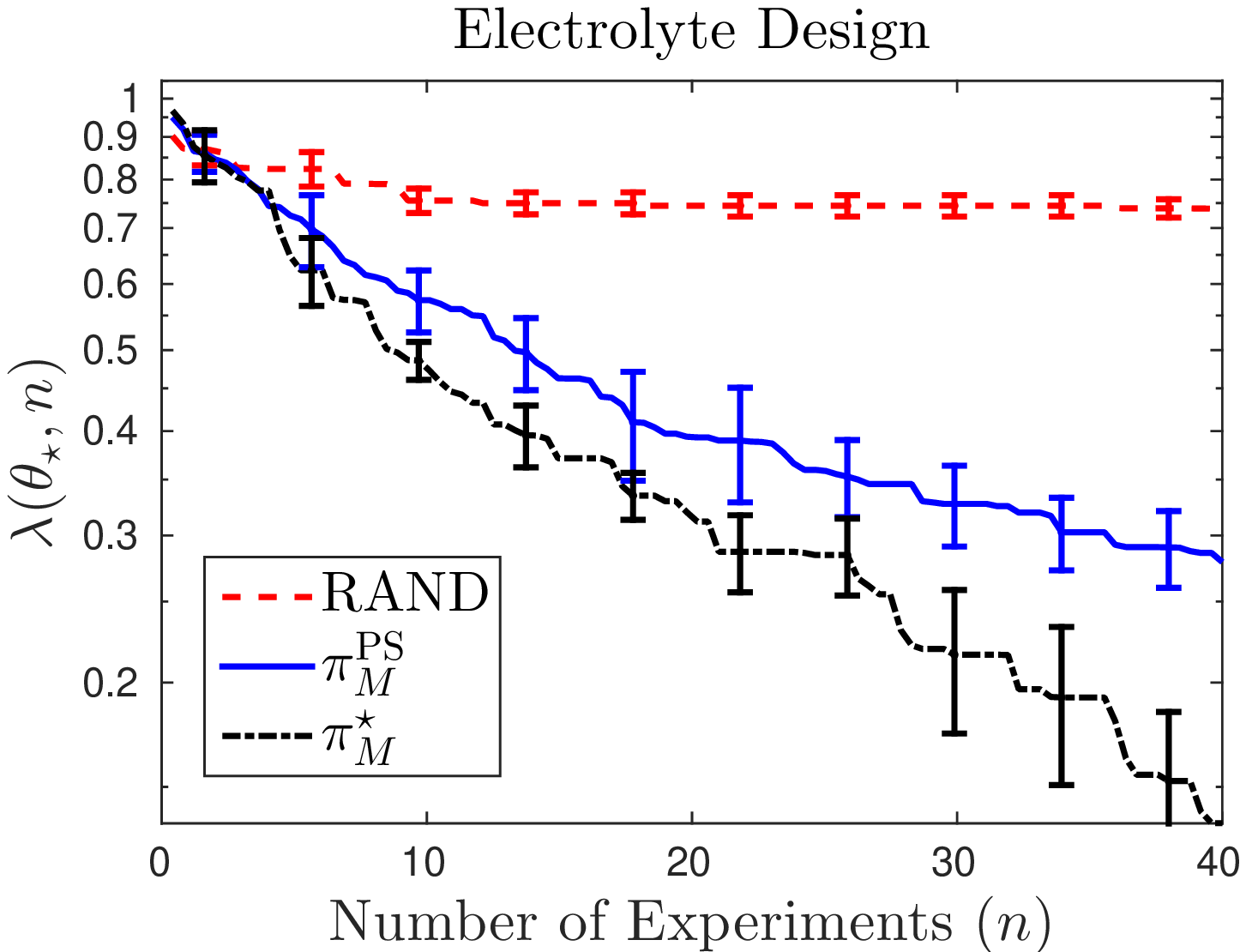}
\label{fig:expeg3}}\expfighsp
\\
\insertleftspace
\subfloat{
\includegraphics[width=\expfigwidth]{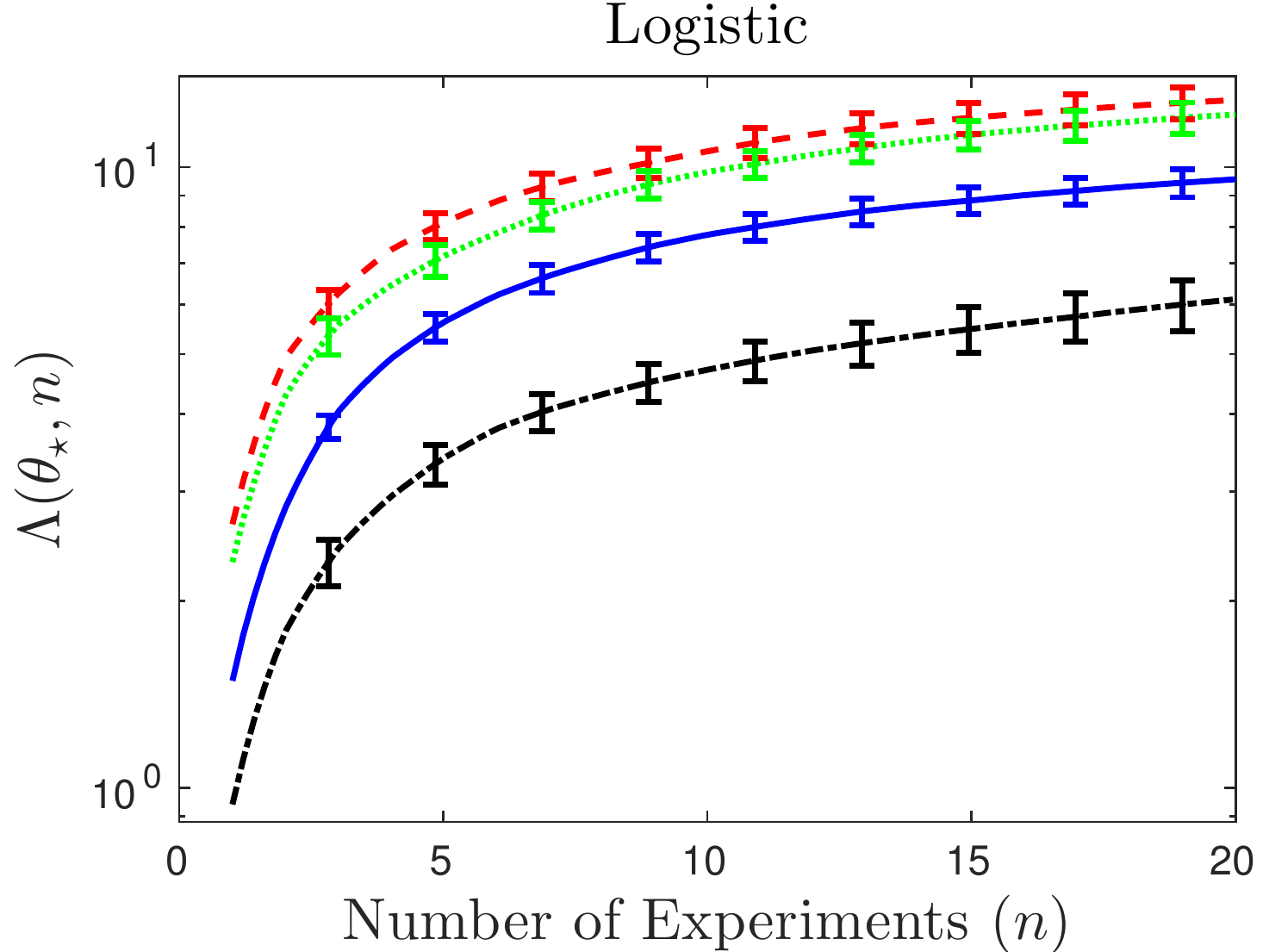}
\label{fig:expeg1}}\expfighsp
\subfloat{
\includegraphics[width=\expfigwidth]{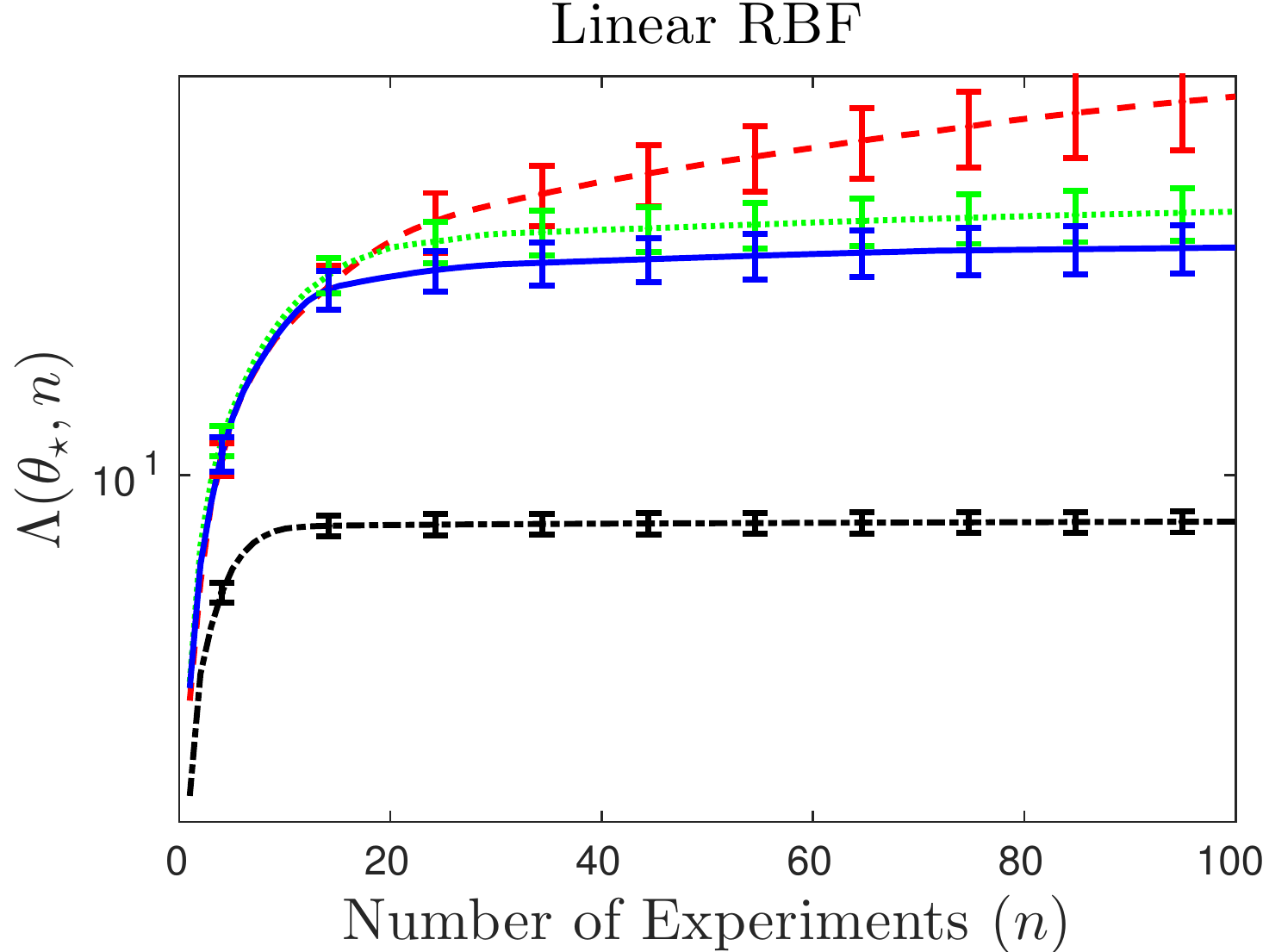}
\label{fig:expeg2}}\expfighsp
\subfloat{
\includegraphics[width=\expfigwidth]{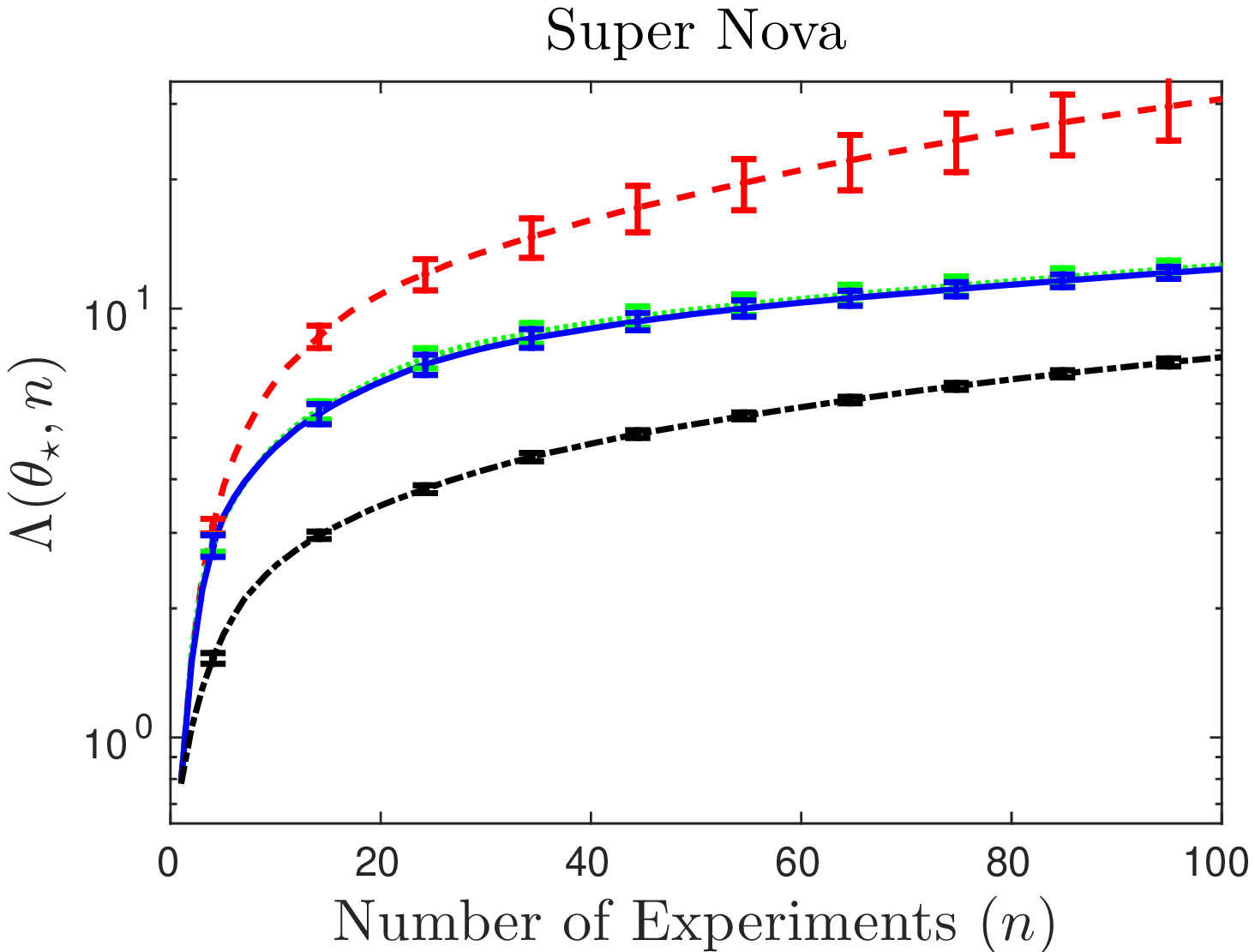}
\label{fig:expeg3}}\expfighsp
\subfloat{
\includegraphics[width=\expfigwidth]{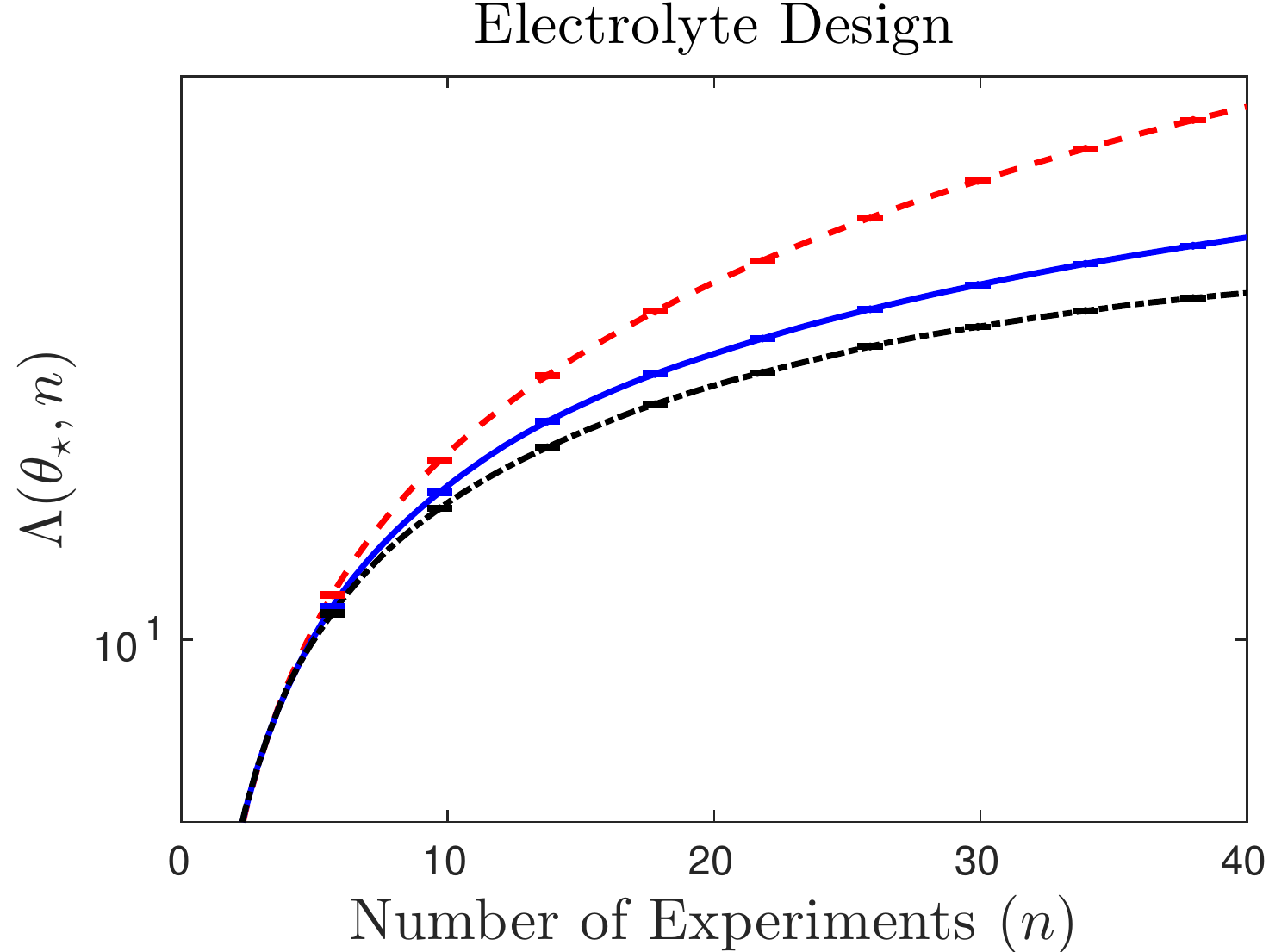}
\label{fig:expeg3}}\expfighsp
% \vspace{-0.10in}
\caption{
\small 
Results on the experiments.
In all figures, the $x$ axis is the number of experiments $n$.
In the top four figures, the $y$ axis is the final penalty $\penl(\thetatrue, n)$
at the $n$\ssth iteration and in the bottom figures,
it is the cumulative penalty $\sumpenl(\thetatrue, n)$.
Lower is better in both cases.
The first two columns are the active learning problems (Sec.~\ref{sec:exal}),
the third is the posterior estimation problem (Sec.~\ref{sec:expost}),
and the fourth is the combined objective problem (Sec.~\ref{sec:excomb}).
All curves were averaged over at least 10 runs, and error bars indicate one
standard error.
\label{fig:exp}
\vspace{-0.10in}
}
\end{center}
\end{figure*}
}

\begin{abstract}
We design a new myopic strategy for a wide class of sequential design of
experiment (DOE) problems, where the goal is to collect data in order to to
fulfil a certain problem specific goal. Our approach, Myopic Posterior Sampling
(\mps), is inspired by the classical posterior (Thompson) sampling algorithm
for multi-armed bandits and leverages the flexibility of probabilistic
programming and approximate Bayesian inference to address a broad set of
problems. Empirically, this general-purpose strategy is competitive with more
specialised methods in a wide array of DOE tasks, and more importantly, enables
addressing complex DOE goals where no existing method seems applicable.  On the
theoretical side, we leverage ideas from adaptive submodularity and
reinforcement learning to derive conditions under which \mpss achieves
sublinear regret against natural benchmark policies.
\end{abstract}

\section{Introduction}
\label{sec:intro}

%To mention,
%\begin{enumerate}
%\item differentiate between OL, RL, submodularity settings
%\item why  myopic strategies
%\end{enumerate}

%% Related work
%% \begin{enumerate}
%% \item Connectio
%% \end{enumerate}

%To cite:
%~\citet{chen2017near,gabillon2013adaptive,gabillon2014large}

%\towillie{In addition to the stuff that we discussed today,
%\begin{enumerate}
%\item How PP enables us to do this.
%\item How many problems in science and industry have very specific objectives:
%not just vanilla AL or BO.
%Hence, we need a flexible framework to do this.
%\item Why we need to incorporate domain expertise in the form of a prior.
%\end{enumerate}
%}

Many real world problems fall into the design of
experiments (DOE) framework, where one wishes to design a sequence of
experiments and collect data so as to achieve a desired goal.
% Many industrial and scientific problems can be treated in the 
% design of experiments (DOE) framework, where the goal is to design a sequence of
% experiments and collect data so as to achieve a desired goal.
% so as to minimise a given penalty $\penl$.
% In many practical problems of interest $\penl$
For example, in electrolyte design for batteries, a chemist would like
to conduct experiments that measure battery conductivity
in order to identify an electrolyte that maximises the conductivity.
On a different day,
she would like to conduct experiments with different electrolyte designs to
learn how the viscosity of
the electrolyte changes with design.
These two tasks,
black-box optimisation and active learning, fall under the umbrella of
DOE and are pervasive in industrial and scientific applications.

% along others fall under
% the umbrella of DOE,

% are pervasive in industrial and
% scientific problems

% Some common examples are optimisation and active learning.
% In the former, each experiment can be interpreted as an evaluation of some black-box
% function $f$ and the goal is to find a design that maximises $f$.
% For example, in electrolyte design for batteries, a chemist wishes to identify
% an electrolyte that maximises the conductivity.
% In active learning, an experiment at  $X$ can be interpreted as a draw from the
% discriminative distribution $\YX|X,\theta$; our goal is to obtain data in the form
% of $(X,\YX)$ pairs and estimate the parameter $\theta$ from this data.
% For instance, the above battery chemist may wish to conduct experiments at
% different electrolyte designs, in order to quickly learn how the viscosity of
% the electrolyte changes with the design.

While several methods exist for specific DOE tasks, real world
problems are broad and complex, and specialised approaches have
limited applicability. Continuing with the electrolyte design example,
the chemist can typically measure both conductivity and viscosity with
a single experiment~\citep{gering2006prediction}.
Since such experiments are expensive, it is 
wasteful to first perform  a set of experiments to optimise conductivity 
and then a fresh set to learn viscosity. 
It is preferable to design a single set of experiments
that simultaneously achieves both goals.
Another example is metallurgy, where one wishes to conduct experiments
to identify phase transitions in an alloy as the composition of metals
changes~\citep{bunn2016semi}.
Here and
elsewhere, both the model and the goal of the experimenter are very 
application specific and cannot be simply shoe-horned into
formulations like black-box optimisation or active learning.

% While there are several methods for specific applications such as the above,
% real world DOE tasks are more broad and complex.
% For example, in the above electrolyte design setting, we might be able to measure
% both the conductivity and viscosity via a single experiment~\citep{gering2006prediction}.
% Since each experiment is expensive, it is inefficient to first
% perform a set of experiments to optimise conductivity and then a second
% set to learn viscosity.
% Rather, the chemist prefers to design a single sequence of experiments
% that simultaneously achieves both
% objectives.
% Similarly, in metallurgy~\citep{bunn2016semi}, one wishes to conduct experiments
% to identify the phase transitions in the alloy as the composition of different metals
% change.
% The goal in these and similar tasks are very application
% specific and cannot be simply shoe-horned
% into broad categories such as optimisation or active learning.
To address these varied applications, we develop a
general and flexible framework for DOE, where a practitioner may
incorporate domain expertise about the system via a Bayesian model and
specify her desired goal via a penalty function $\penl$, which can
depend on unknown system characteristics and the data collected during the
DOE process. We then develop a myopic strategy for DOE, inspired by
posterior (Thompson) sampling for multi-armed
bandits~\cite{thompson33sampling}. 
Our approach has two key advantages. First, the Bayesian formulation
allows us to exploit advances in probabilistic
%programming~\citep{neiswanger2015embarrassingly,tran2017edward} to
programming~\citep{carpenter2017stan,tran2017edward} to
incorporate domain expertise without introducing complexity. Since
experiments are typically extremely expensive in applications,
incorporating domain
expertise is essential to achieving the desired goal in few experiments.
Probabilisitic programming offers an elegant method to do so.
Second, our myopic/greedy strategy is simple and
computationally attractive in comparison with policies that engage in
long-term planning.  Nevertheless, borrowing ideas from submodular
optimisation and reinforcement learning,
we derive natural conditions under which our myopic
policy is competitive with the globally optimal one.
Our specific contributions are: % n this work are as follows.
\vspace{-0.03in}
\begin{enumerate}[leftmargin=0.2in]
\item We propose a flexible framework for DOE that allows a practitioner to
  describe their system (via a probabilistic model) and specify their
  goal (via a penalty function).
% state their desired goal in the form of a penalty function.
We also derive an algorithm, Myopic Posterior Sampling (\mps), for this setting.
\vspace{-0.03in}
% which is inspired by posterior sampling.
\item 
We implement \mpss using probabilistic programming and
demonstrate that it performs favourably in a variety of synthetic and real
world DOE problems.
Despite our general formulation, \mpss is competitive with specialised methods designed
for particular problems.
% being a generic framework,
% it performs favourably with methods specifically designed for certain problems.
\vspace{-0.03in}
% \item In our theoretical analysis, we study conditions under which a myopic policy
% may achieve reasonable performance when it does not know or observe the
%  penalty.
% \mpss can be competitive with myopic and globally optimal strategies
% which know the penalty under suitable assumptions.
\item In our theoretical analysis, we explore conditions under which
  \mps, which learns about the system over time, is competitive with
  myopic and globally optimal strategies that have full knowledge of
  the system.
% a myopic policy
% which does not know or observe the penalty, such as \mps,
% may be competitive with myopic and globally optimal strategies
% which know the penalty.
\end{enumerate}

\textbf{Related work:} The classical results for (sequential) DOE
focus on discrete
settings~\cite{robbins1952some,chernoff1972sequential} or linear
models~\cite{fedorov1972theory}, which enable a more detailed
characterization and refined analysis than we provide. More recent
work in the bandit community studies more complex non-linear
models~\cite{bubeck2012regret,srinivas10gpbandits,streeter2009online},
but ignores temporal dependencies that arise in applications.
% Extending these results, our work
% considers complex models where the agent's decisions affect the future
% state, which is essential in our applications of interest.
% 
We focus on posterior sampling (PS)~\cite{thompson33sampling} as the
bandit algorithm, since it has proven to be quite general and admits a
clean Bayesian analysis~\cite{russo2016information}.  PS has been
studied in a number of bandit
settings~\cite{gopalan2014thompson,kawale2015efficient,kandasamy2017asynchronous},
and some episodic RL
problems~\cite{osband2013more,osband2014near,gopalan2015thompson},
where the agent is allowed to restart. In contrast, here we study PS
on a single long trajectory with no restarts.

Myopic/greedy policies are known to be near-optimal for sequential
decision making problems with \emph{adaptive
  submodularity}~\cite{golovin2011adaptive}, which generalizes
submodularity and formalizes a diminishing returns property.  Adaptive
submodularity has been used for several DOE setups including active
learning~\cite{golovin2010near,chen2013near,chen2017near} and
detection~\cite{chen2014active}, but these papers focus on
characterizing applications that admit near-optimal greedy strategies,
and do not address the question of \emph{learning} such a policy.
% one need not
% learn a near-optimal policy as we consider here.
% For sequential decision making problems with partial observability,
% which we are considering here, adaptive submodularity is a structural
% condition---generalizing submodularity and formalizing diminishing
% returns---under which greedy adaptive policies are known to be near
% optimal. 
As such, these results are complementary to ours: adaptive
submodularity controls the approximation error (the difference between
myopic- and globally-optimal strategies, both of which know the
penalty $\penl$), while we control the estimation error (how close our
learned policy is to the myopic optimal policy that knows $\penl$).
As we show in Theorem~\ref{thm:submodthm}, with adaptive
submodularity, \mpss can also compete with the globally optimal
non-myopic policy.  Prior results for learning in (adaptive)
submodular environments are episodic and allow
restarts~\cite{gabillon2013adaptive,gabillon2014large}, which is
unnatural in the DOE setup.

% Since we are focusing on myopic/greedy strategies, a natural point of
% comparison is the literature on (adaptive) submodular
% optimization. 

% Similarly, most results for learning submodular functions are episodic. 

% The classical results focus on computational issues;
% efficiently optimizing a known submodular function via greedy methods,
% but recent work has addressed various learning problems where the
% function is unknown, as we do here. On the other hand, the results for
% learning submodular functions are episodic, while we focus on learning
% from a single long trajectory (which necessarily requires further
% structure). \akshay{On the other hand, we are not aware of any
%   ``bayesian'' methods for bandit submodular optimization, which we
%   provide here for the first time.}

Our formulation can also be cast as reinforcement learning since at
each round the agent makes a decision (what experiment to perform)
with the goal of minimizing a long-term cost (the penalty
function).
 One goal of our work is to understand when myopic
``bandit-like" strategies perform well in reinforcement learning
environments with long-term temporal dependencies. 
There are two main differences with prior
work~\citep{kearns2002near,strehl2009reinforcement,jaksch2010near,osband2013more,osband2014near}:
first, we make no explicit assumptions about the complexity of the
state and action space, instead placing assumptions on the penalty (reward)
structure and optimal policy, which is a better fit for our
applications. More importantly, in our setup, the true penalty is never
revealed to the agent, and instead it receives side-observations that
provide information about an underlying parameter governing the
environment. Lastly, our focus is on understanding when myopic
strategies have reasonable performance rather than on achieving global
optimality; it may be possible and interesting to extend these results
to the general RL setting.

\section{Set up and Method}
\label{sec:method}

Let $\Theta$ denote a parameter space, $\Xcal$ an action space, and
$\Ycal$ an outcome space.  We consider a Bayesian setting where a
\emph{true parameter} $\thetatrue \in \Theta$ is drawn from a prior
distribution $\pthetazero$. A decision maker repeatedly chooses an
action $X \in \Xcal$, conducts an experiment at $X$, and observes the
outcome $\YX \in\Ycal$. We assume $\YX$ is drawn from a
\emph{likelihood} $\PP(\cdot|X,\thetatrue)$, with known distributional
form. This process proceeds for $n$ rounds, resulting in a
\emph{data
sequence} $\dataseqn = \{(\Xj,\YXj)\}_{j=1}^n$, which is an ordered
multi-set of action-observation pairs. With $\Dcalseq$ denoting the
set of all possible data sequences, the goal is to minimise a
\emph{penalty function} $\penl: \Theta\times\Dcalseq\rightarrow[0,1]$.
In particular, we focus on the following two criteria, depending
on the application:
%  that depends
% both on the true parameter and the collected data sequence. Depending
% on the application, we focus on the following two quantities:
\begin{align*}
\cridx\;\;
\sumpenl(\thetatrue, \datan) = \sum_{t=1}^n\penl(\thetatrue, \datat)
\hspace{1.0in}
% {\rm \bf (ii)}\;\;
\fridx\;\;
\penl(\thetatrue, \datan),
\label{eqn:penldefn} \numberthis
\end{align*}
Here, $\datat=\{(\Xj,\YXj)\}_{j=1}^t$
denotes the prefix of length $t$ of the data sequence
$\datan$ collected by the decision maker.
The former notion is the
cumulative sum of all penalties, while the latter corresponds just to
the penalty once all experiments have been completed. Note that since
the penalty function depends on the unknown true parameter
$\thetatrue$, the decision maker cannot compute the penalty during the
data collection process, and instead must infer the penalty from
observations in order to minimise it. This is a key distinction from
existing work on reinforcement learning and sequential optimisation,
and one of the new challenges in our setting.

\insertprethmspacing
\begin{example}
  A motivating example is \emph{Bayesian active
    learning}~\cite{golovin2010near,chen2017near}. Here, actions $X$
  correspond to data points while $\YX$ is the label and
  $\PP(y|x,\theta)$ specifies an assumed discriminative model. We
  may set
  $\penl(\theta,\datan) = \|\tau(\theta)-\hat{\tau}(\datan)\|_2^2$
  where $\tau$ is a parameter of interest and $\hat{\tau}$ is a
  predetermined estimator (e.g. maximum likelihood or maximum a posteriori).
  The true penalty
  $\penl(\thetatrue,\datan)$ is
  not available to the decision maker since it requires knowing
  $\tau(\thetatrue)$.
\end{example}

\paragraph{Notation:} 
For each $t \in \NN$, let
$\Dcalseqt = \{\{(\Xj,\YXj)\}_{j=1}^t: \Xj\in\Xcal, \YXj\in\Ycal\}$
denote the set of all data sequences of length $t$, so that
$\Dcalseq = \bigcup_{t\in\NN}\Dcalseqt$.  We use $|D|$ to denote the
length of a data sequence and $D \concat D'$ for the concatenation of
two sequences.  $D \dssubset D'$ and $D' \dssupset D$ both equivalently denote
that $D$ is a prefix of $D'$. Given a data sequence $\dataseqt$, we
use $\dataseqtt{t'}$ for $t' < t$ to denote the prefix of the first
$t'$ action-observation pairs.

A \emph{policy} for experiment design chooses a sequence of actions
$\{\Xj\}_{j\in\NN}$ based on past actions and observations.
% where $\Xt$ can depend on
% past actions and observations, $\dataseqtmo = \{(\Xj,\YXj)\}_{j=1}^{t-1}$.
In particular, for a \emph{randomised} policy
$\policy = \{\policyj\}_{j\in\NN}$, at time $t$, an action is drawn
from $\policyt(\dataseqtmo) = \PP(\Xt\in\cdot|\dataseqtmo)$. 
Two policies that will appear frequently in the sequel are $\gopolicy$ and
$\uopolicy$, both of which operate with knowledge of
$\thetatrue$. $\gopolicy$ is the myopic optimal policy, which, from
every data sequence $\datat$ chooses the action $X$ minimizing the expected penalty
at the next step: $\EE[\penl(\thetatrue,\datat\concat\{(X,\YX)\})|\thetatrue, \datat]$.
On
the other hand $\uopolicy$ is the non-myopic, globally optimal adaptive policy,
which in state $\datat$ with $n-t$ steps to go chooses the action to
minimise the expected long-term penalty:
$\EE[\penl(\thetatrue,\datat\concat\{(X,\YX)\}\concat \datatt{t+2:n}) \mid
\uopolicy,\thetatrue, \datat]$.
Observe that $\uopolicy$ may depend on the time horizon $n$ while
$\gopolicy$ does not.

\subsection*{Design of Experiments via Posterior Sampling}

We present a simple and intuitive myopic strategy that aims to minimise $\penl$ based on
the posterior of the data collected so far.
For this, first define the expected look-ahead penalty %at time $t$,
$\penlplus:\Theta\times\Dcal\times\Xcal\rightarrow[0,1]$ to be the
expected penalty at the next time step if $\theta\in\Theta$ were the true parameter and
we were to take action $x\in\Xcal$.
Precisely, for a data sequence $D$,
\begin{align*}
\penlplus(\theta,D,x) \;=\; \EE_{\Yx\sim\PP(Y|x,\theta)}
\Big[\penlt\big(\theta,\,D\concat\{(x, \Yx)\}\,\big) \Big].
\numberthis
\label{eqn:lookaheadpenalty}
\end{align*}
The proposed policy, presented in Algorithm~\ref{alg:bdoe},
is called \bdoes (Myopic Posterior Sampling) and is denoted $\pspolicy$.
At time step $t$,
it first samples a parameter value $\theta$ from the posterior  for
$\thetatrue$ conditioned on the data, i.e. $\theta\sim\PP(\thetatrue|\datatmo)$.
Then, it chooses the action $\Xt$ that is expected to minimise the penalty $\penl$ by
pretending that $\theta$ was the true parameter.
It performs the experiment at $\Xt$, collects the observation $\YXt$, and proceeds to the 
next time step.
% keeps repeating
% this routine on each iteration.

\insertAlgoMain

\textbf{Computational considerations:}
It is worth pointing out some of the computational considerations in
Algorithm~\ref{alg:bdoe}.
First, sampling from the posterior for $\thetatrue$ in step~\ref{step:sampletheta}
might be difficult, especially in complex Bayesian models.
Fortunately however, the field of Bayesian inference has made great strides in the recent
past seeing the development of fast techniques for approximate inference methods such as
MCMC or variational inference~\citep{neiswanger2015embarrassingly,hensman2012fast}.
Moreover, today we have efficient probabilistic programming
tools~\citep{carpenter2017stan,tran2017edward} that allow a practitioner to intuitively
incorporate domain expertise via a prior and obtain the posterior given data.
Secondly, the minimisation of the look ahead penalty in step~\ref{step:choosext}
can also be non-trivial, especially since it might involve empirically computing
the expectation in~\eqref{eqn:lookaheadpenalty}.
This is similar to existing work in Bayesian optimisation which assume access to such an
optimisation
oracle~\citep{srinivas10gpbandits,bull11boRates}.
That said, in many practical settings where experiments are  financially expensive and
can take several hours, these considerations are less critical.
% can take hours and involve

% , these considerations may not be significant.
% For example, in many industrial DOE problems, an experiment might require several
% hours and involve significant monetary cost.
% In such settings, one cares more about incorporating domain expertise in a possibly
% complex model, over developing computationally efficient algorithm

Despite these concerns, it is worth mentioning that myopic strategies are
still computationally far more attractive than policies which try to behave globally
optimally.
For example, extending \mpss to a $k$ step look-ahead might involve an optimisation
over $\Xcal^k$ in step~\ref{step:choosext} of Algorithm~\ref{alg:bdoe} which might
be impractical for large values of $k$ except in the most trivial settings.

\textbf{Specification of the prior:}
In real world applications, the prior could be specified by a domain expert with knowledge
of the given DOE problem.
In some instances, the expert may only be able to specify the relations between the
various variables involved.
In such cases, one can specify the parametric form for the prior, and learn the
parameters of the prior in an adaptive data dependent fashion using maximum likelihood
and/or maximum a posteriori techniques~\citep{snoek12practicalBO}.
While we adopt both approaches in our experiments, we assume a fixed prior in our
theoretical analysis.

\section{Examples \& Experiments}
\label{sec:experiments}

In this section, we give some concrete examples of DOE problems that
can be specified by a penalty
function $\penl$ and present experimental results for these settings.
We compare $\pspolicy$ to random sampling (RAND), the myopically 
optimal policy $\gopolicy$ which assumes access to
$\thetatrue$,
and in some cases to specialised methods
developed for the particular problem.
% generally outperformed
% by the non-realisable $\gopolicy$,
% it still remains competitive.
In the interest of aligning our experiments with our theoretical
analysis, we compare methods on both criteria in~\eqref{eqn:penldefn},
although in these applications,
% one is more interested in 
the final
penalty
$\penl(\thetatrue, \datan)$ is more important than the cumulative one $\sumpenl(\thetatrue, \datan)$.
% step~\ref{step:sampletheta} of Algorithm~\ref{step:sampletheta}.
% While $\

\paragraph{High-level Takeaways:}
% \textbf{Implementation details:}
% Among the above examples, exact posterior inference was possible in the linear Gaussian
% model.
Despite being a quite general,
$\pspolicy$ outperforms, or performs as well as, specialised
methods.
$\pspolicy$ is competitive, but slightly worse than the non-realisable
$\gopolicy$. 
Finally
$\pspolicy$ enables effective DOE in complex settings where no prior
methods seem applicable.

\paragraph{Implementation details:}
One of the experiments in Section~\ref{sec:exal} admits analytical computation
of the posterior.
In all other experiments, we use the Edward probabilistic programming
framework~\citep{tran2017edward}.
We use variational inference to approximate the
posterior, and then draw a sample from this approximation.
The look-ahead
penalty~\eqref{eqn:lookaheadpenalty} is computed empirically by
drawing $50$ samples from $Y|X,\theta$ for the sampled $\theta$.  We
minimise $\penlplus$ by evaluating it on a fine grid and choosing the
maximum.  We use grid sizes $100$, $2500$, and $27000$ respectively
for one, two and three dimensional domains $\Xcal$.

\subsection{Active Learning}
\label{sec:exal}

\textbf{Problem:}
As described previously,
we wish to learn some parameter $\tautrue = \tau(\thetatrue)$ which
is a function of the true parameter $\thetatrue$.
Each time we query some $X\in\Xcal$, we see a noisy observation (label)
$Y\sim\PP(Y|X,\thetatrue)$.
We conduct two synthetic experiments in this setting.
We use $\|\tautrue - \tauhat(\datan)\|_2^2$ as the penalty where
$\tauhat$ is a regularised maximum likelihood estimator.
In addition to RAND and $\gopolicy$, we compare $\pspolicy$ to the ActiveSelect
method of~\citet{chaudhuri2015convergence}.

\textbf{Experiment 1:}
We use the following logistic regression model:
$\Yx|x,\theta \sim \Ncal(\ftheta(x), \eta^2)$ where
$\ftheta(x) =  \frac{a}{1 + e^{b(x-c)}}$.
% Here $a,b,c,\eta^2>0$.
The true parameter is $\thetatrue = (a, b, c, \eta^2)$ and our goal is to
estimate $\tautrue = (a, b, c)$.
The MLE is computed via gradient ascent on the log likelihood.
In our experiments, we used $a=2.1, b=7, c=6$ and $\eta^2 = 0.01$ as $\thetatrue$.
We used normal priors $\Ncal(2, 1)$, $\Ncal(5, 3)$ and $\Ncal(5, 3)$ for
$a, b, c$ respectively and an inverse gamma ${\rm IG}(20, 1)$ prior for
$\eta^2$.
As the action space, we used $\Xcal = [0, 10]$.
For variational inference, 
we used a normal approximation for the posterior for $a, b, c$ and
an inverse gamma approximation for $\eta^2$.
The results are given in the first column of Figure~\ref{fig:exp}.

\insertFigExpone

\textbf{Experiment 2:}
In the second example, we use the following linear regression model:
$\Yx|x,\theta \sim \Ncal(\ftheta(x), 0.01)$ where
$\ftheta(x) = \sum_{i=1}^{16} \theta_{*i} \phi(x-c_i)$.
Here, $\phi(v) = \frac{1}{\sqrt{0.2\pi}} e^{-5\|v\|_2^2}$ and
the points $c_1, \dots, c_{16}$ were arranged in a $4\times 4$ grid
within $[0, 1]^2$.
We set $\theta_{*i} = g(c_i)$, with $g(v) = \sin(3.9 \pi ((v_1-0.1)^2 + v_2 + 0.1))$.
% We used $\eta^2 = 0.01$, and treated it as a known variable.
Our goal is to estimate $\tautrue = \thetatrue$.
As the action space, we used $\Xcal = [0, 1]^2$.
% Regularized least-squares is used to compute $\tauhat(\datan)$.
The posterior for $\thetatrue$ was calculated in closed form
% according to Bayesian Linear Regression,
using a normal distribution $\Ncal(0, I_{16})$ as the prior.
The results are given in the second column of Figure~\ref{fig:exp}.

\subsection{Posterior Estimation \& Active Regression}
\label{sec:expost}

\textbf{Problem:}
Consider estimating a non-parametric function $\fthetatrue$, which is known to be uniformly
smooth.
An action $x\in\Xcal$ is a query to the function $f$, upon which we observe
$\Yx = \fthetatrue(x) + \epsilon$, where $\EE[\epsilon] = 0$.
% While adaptive 
% Adaptive methods are known to not do significantly better than
If the goal is to learn $\fthetatrue$ uniformly well in $L^2$ error,
i.e. with penalty 
$\|\fthetatrue - \fhat(\datan)\|^2$,
adaptive techniques may not perform
significantly better than non-adaptive ones~\citep{willett2006faster}.
However, if our penalty was $\penl(\thetatrue, \datan) =
\|\sigma(\fthetatrue) - \sigma(\fhat(\datan))\|^2$
for some monotone super-linear transformation $\sigma$,
then adaptive techniques may do %significantly
better by requesting more evaluations at regions with high $\fthetatrue$ value.
This is because, $\penl(\thetatrue, \datan)$ is more sensitive to such regions
due to the transformation $\sigma$.

A particularly pertinent instance of this formulation arises in
% example for this setting occurs in
astrophysical applications where one wishes to estimate the posterior
distribution
of cosmological parameters, given some astronomical
data $Q$~\citep{parkinson06wmap3}.
Here, an astrophysicist specifies a prior $\Xi$ over
 the cosmological parameters $Z\in\Xcal$, and 
the likelihood of the data for a given choice 
of the cosmological parameters $x\in\Xcal$  is computed via 
an expensive astrophysical simulation. 
The prior and the likelihood gives rise to an unknown log joint density\footnote{
It is important not to conflate the astrophysical Bayesian model 
which specifies a prior over $\Xcal$ 
with our algorithm which assumes a prior over $\Theta$.
}  $\fthetatrue$ defined on $\Xcal$,
% This, along with $\Xi$ gives
% rise to the log joint density
and the goal is to estimate the 
% posterior density  $p(Z|Q)$ or equivalently
the joint density $p(Z=x, Q) = \exp(\fthetatrue(x))$ so that we can perform posterior inference.
%  which is a constant 
% factor away from $p(Z|Q)$.
Adopting assumptions from prior work~\citep{kandasamy2015bayesian} we
model $\fthetatrue$ as a Gaussian process, which is reasonable since
we expect a log density to be smoother than the density itself.
% In this setting, one can model
%  $\fthetatrue$ as a Gaussian process~\citep{kandasamy2015bayesian}, since log
% densities are generally smoother than densities and are amenable to be modeled as such.
As we wish to estimate the joint density, $\penl$
takes the above form with $\sigma = \exp$.

\textbf{Experiment 3:}
We use data on Type I-a supernova from~\citet{davis07supernovae}.
We wish to estimate the posterior over the Hubble constant
$H \in (60, 80)$, the dark matter fraction $\Omega_M\in(0,1)$
and the dark energy fraction $\Omega_E\in(0,1)$, which constitute our
three dimensional action space $\Xcal$.
The likelihood is computed via the Robertson-Walker metric.
In addition to $\gopolicy$ and RAND, we compare $\pspolicy$ to Gaussian
process based exponentiated variance reduction
(GP-EVR)~\citep{kandasamy2015bayesian}
which was specifically designed for this setting.
We evaluate the penalty via numerical integration.
The results are presented in the third column of Figure~\ref{fig:exp}.

\subsection{Combined and Customised Objectives}
\label{sec:excomb}

\textbf{Problem:}
In many real world problems, one needs to design experiments with multiple goals.
For example, an experiment might evaluate multiple objectives, and the task might
be to optimise some of them, while learning the parameters for another.
Classical methods specifically designed for active learning or optimisation may not
be suitable in such settings.
One advantage to the proposed framework is that it allows us to combine multiple goals
in the form of a penalty function.
For instance, if an experiment measures two functions $\fthetatrueo, \fthetatruet$
and we wish to
learn $f_1$ while optimising $f_2$, we can define the penalty as
$\penl(\thetatrue, \datan) =  
\|\fthetatrueo - \fhato(\datan)\|^2 + (\max \fthetatruet -
\max_{\Xt, t\leq n}\fthetatruet(\Xt))$.
Here $\fhato$ is an estimate for $\fhato$ obtained from the data,
$\|\cdot\|$ is the $L_2$ norm and 
$\max_{\Xt, t\leq n}\fthetatruet(\Xt)$ is the maximum point of
$\fthetatruet$ we have evaluated so far.
Below, we demonstrate one such application.

\textbf{Experiment 4:}
In battery electrolyte design, one tests an electrolyte composition under various
physical conditions.
On an experiment at $x\in\Xcal$,
we obtain measurements $\Yx = (\Yxsolv, \Yxvisc, \Yxcond)$
which are noisy measurements
 of the solvation energy  $\fsolv$, 
the viscosity $\fvisc$ and the specific conductivity $\fcond$.
Our goal is to estimate $\fsolv$ and $\fvisc$ while optimising $\fcond$.
Hence,
\[
\penl(\thetatrue, \datan) =  
\alpha \|\fsolv - \fsolvhat(\datan)\|^2 + 
\beta \|\fvisc - \fvischat(\datan)\|^2 +
\gamma (\max \fcond -
\max_{\Xt, t\leq n}\fcond(\Xt)),
\]
where, the parameters $\alpha, \beta, \gamma$ were chosen so as to scale each objective
and ensure that none of them dominate the penalty.
In our experiment,
we use the dataset from~\citet{gering2006prediction}.
Our action space $\Xcal$ is parametrised by the following three variables:
$Q\in(0,1)$ measures the proportion of two solvents EC and EMC in the electrolye,
$S\in(0, 3.5)$ is the molarity of the salt $\text{LiPF}_6$
and $T\in(-20, 50)$ is the temperature in Celsius.
We use the following prior which is based off a physical understanding
of the interaction of these variables.
$\fcond:\Xcal\rightarrow\RR$ is sampled from a Gaussian process (GP),
$\fvisc(Q, S, T) = \exp(-a T)\gvisc(Q, S)$ where $\gvisc$ is sampled from a  GP,
and 
$\fsolv(Q, S, T) = b + \exp(cQ - dS - eT)$.
We use inverse gamma priors for $a, b, d, e$ and a normal prior for $c$.
For variational inference, we used inverse gamma approximations for $a,b,d,e$,
a normal approximation for $c$,
and GP approximations for $\fcond$ and $\gvisc$.
We use the posterior mean of $\fsolv$ and $\fvisc$ under this prior
as the estimates $\fsolvhat, \fvischat$.
We present the results in the fourth column of Figure~\ref{fig:exp} where 
we compare RAND, $\pspolicy$ and $\gopolicy$.
This is an example of a customised DOE problem for which no prior method seems
directly applicable. 
% one cannot use any existing
% methods.

\subsection{Bandits \& Bayesian Optimisation}

Lastly, we mention that bandit optimisation is a self-evident special case of our
formulation.
% application for the above setting is bandit optimisation.
Here, the parameter $\thetatrue$ specifies a function $\fthetatrue:\Xcal\rightarrow\RR$.
When we choose a point $X\in\Xcal$ to evaluate the function, we observe
$\YX = \fthetatrue(X) + \epsilon$ where $\EE[\epsilon] = 0$.
In the bandit framework, the penalty is the instantaneous regret
$\penl(\thetatrue, \datan) = \max_{x\in\Xcal}\fthetatrue(x) - \fthetatrue(\Xn)$.
% where $\Xn$ is the last point that was chosen in $\datan$.
In Bayesian optimisation, one is interested in simply finding a single
value close to the optimum and hence
$\penl(\thetatrue, \datan) = \max_{x\in\Xcal}\fthetatrue(x) - 
\max_{t\leq n}\fthetatrue(\Xt)$.
%  where $\{\Xt\}_{t=1}^n$ are the previous evaluations
% in $\datan$.
In either case, $\pspolicy$ reduces to the Thompson sampling procedure
as $\argmin_{x\in\Xcal} \penlplus(\thetatrue, \datatmo, x)$
$= \argmax_{x\in\Xcal} \ftheta(x)$, where $\ftheta$ is a random
function drawn from the posterior.  Since prior work has demonstrated
that Thompson sampling performs empirically well in several bandit
optimisation
settings~\citep{chapelle2011empirical,hernandez2017parallel,kandasamy2018parallel},
we omit experimental results for this example. 
One can also cast other
variants of Bayesian optimisation, including multi-objective
optimisation~\citep{hernandez2016predictive} and constrained
optimization~\citep{gardner2014bayesian}, in our general formulation.
% via such
% a penalty function. 

\section{Theoretical Analysis}
\label{sec:regret}

% We will define two notions of regret, corresponding to the two criteria
% in~\eqref{eqn:penldefn}, after collecting a data sequence
% $\datan$ in $n$ experiments.
In this section we derive theoretical guarantees for $\pspolicy$. Our
emphasis is on understanding conditions under which myopic learning
algorithms can perform competitively with the myopic optimal strategy
$\gopolicy$ and even the globally optimal strategy $\uopolicy$
(see Section~\ref{sec:method}).

Let the loss $\lossn(\thetatrue,\policy)$ of a policy $\policy$ after
$n$ evaluations be the expected sum of cumulative penalties for fixed
$\thetatrue$, i.e.
$\lossn(\thetatrue,\policy) =
\EE[\sumpenl(\thetatrue,\dataseqn)|\thetatrue,\policy]$
where $\dataseqn$ is the data collected by $\policy$
(Recall~\eqref{eqn:penldefn}).
For criterion $\cridx$, we are
interested in upper bounding $\lossn(\thetatrue,\policy)$ in terms of
$\lossn(\thetatrue,\gopolicy)$,
which yields a \emph{cumulative regret bound}, and for criterion
$\fridx$, we hope to bound
$\EE[\penl(\thetatrue,\dataseqn)\mid\thetatrue,\policy]$ in terms of
the analogous quantities for $\gopolicy,\uopolicy$, which serves as an
\emph{final regret bound}.
Note that a comparison with
$\uopolicy$ on $\cridx$ is meaningless since it might take high
penalty actions in the early stages in order to do well in the long
run.
Our bounds will hold in expectation
over $\thetatrue\sim\pthetazero$.

% we define the
% regret of a policy $\policy$ after $n$ evaluations with respect to
% $\gopolicy$ as,
% \begin{align*}
% % R(\thetatrue, \policy) &=
%   \lossn(\thetatrue, \policy) - \lossn(\thetatrue, \gopolicy)
%   = \EE[\sumpenl(\thetatrue, \datan)|\thetatrue,\policy] - \EE[\sumpenl(\thetatrue, \godatan)|\thetatrue,\gopolicy].
% %   = \EE\sum_{t=1}^n \penl(\thetatrue, \datat)\, -\, 
% %         \EE\sum_{t=1}^n \penl(\thetatrue, \godatat). %\\
% \label{eqn:regretDefn}\numberthis
% % \frn &= \penln(\thetatrue, \datan)\, -\, \penln(\thetatrue, \ordatan)
% \end{align*}
% Where the expectations are over the observations and any randomness in
% the policies.  This regret is itself a random quantity since
% $\thetatrue$ is random, so we will consider the Bayesian regret
% $\EEtt{\thetatrue\sim\pthetazero}[\lossn(\thetatrue,\policy)
% -\lossn(\thetatrue, \gopolicy)]$, taking expectation over the prior
% $\pthetazero$. For criterion $\fridx$ we measure regret via

% then $R(\thetatrue,n,\pspolicy)$ is sub-linear in $n$.
% The expectation above is over the observations and the randomness of our policy $\pi$.
% The regret in~\eqref{eqn:regretDefn} is itself a random quantity since $\thetatrue$
% is random, so we will consider the
% Bayesian regret $\EEtt{\thetatrue\sim\pthetazero}[R(\thetatrue,n,\policy)]$,
% taking expectation
% over the prior $\pthetazero$.

% \akshay{I suggest stating this as a proposition.} 
The following proposition shows that without further assumptions, a
non-trivial regret bound is impossible. Such results are common in the
RL literature, and motivate several structure assumptions, including
small diameter~\cite{jaksch2010near} and episodic
problems~\cite{dann2015sample,osband2016posterior}.

\insertprethmspacing
\begin{proposition}
  There exists a DOE problem where
  $\lim_{n\to\infty}
  \EEtt{\thetatrue\sim\pthetazero}[\lossn(\thetatrue, \policy) -
  \lossn(\thetatrue, \gopolicy)] = 1/2$.
\end{proposition}
\begin{proof}
  Consider a setting with uniform prior over two parameters
  $\theta_0,\theta_1$ with two actions $X_0,X_1$. Set
  $\lambda(\theta_0,D) = \lambda(\theta_1,\cdot) = \one\{X_{1} \in D\}$.
  %  and similarly for
%   $\lambda(\theta_1,\cdot)$.
  If $\thetatrue=\theta_0$, then
  $\gopolicy$ will repeatedly choose $X_0$ and incur cumulative (and final)
  loss $0$ and similarly when $\thetatrue=\theta_1$.
  On the other
  hand, the first decision for the decision maker must be the same for
  both choices of $\thetatrue$ and hence the regret is $1/2$.
\end{proof}
% We first need to
% acknowledge that achieving good regret without additional assumptions
% is impossible since the penalty is not observed.  For example, assume
% there exists an action $\badaction\in\Xcal$ such that for all data
% sequences $D$ which contain $\badaction$, we have maximum penalty
% $\penl(\thetatrue,D) = 1$.  While $\gopolicy$, which knows
% $\thetatrue$ and hence the penalty, will never choose that action, a
% policy which does not know $\thetatrue$ might do so, especially in the
% early exploration stages.

% Such pathological examples arise in infinite-horizon RL, and motivate
% recoverability or diameter conditions~\cite{jaksch2010near}, which we
% will also require. Another workaround from the RL literatue is to
% consider episodic settings. 

% This is similar to settings in infinite horizon reinforcement learning
% where a single bad action might us to a state from which we cannot recover.
% Most theoretical results for RL are in the episodic
% setting~\citep{dann2015sample,osband2016posterior}, where, even if we were to
% perform poorly in one episode, there is hope to do well in the long run since we start
% afresh at the next episode. 

% \toworkon{Why this is difficult without assumptions. Refer to RL literature.}
Motivated by this lower bound, we will study a variety of conditions
on the penalty function, under which a policy can achieve sub-linear
regret. We consider three such structural conditions, and our
results apply to environments satisfying \emph{any one} of these
conditions.

\insertprethmspacing
\begin{condition}[Structural conditions]
  \label{cond:structural_cond}
  Consider the following three conditions:
  \begin{condenum}
  \item \label{cond:instantaneous} {\bf Episodic Penalties.} There
    exists $H \in \NN$ such that for all $t$ and all $\datat$, we have
    \begin{align*}
      \penl(\thetatrue,\datat) =
  \penl\left(\thetatrue,\{(\Xj,\YXj)\}_{j=t - \lfloor t/H\rfloor}^t\right)
    \end{align*}
    Thus, the penalty at time $t$ depends on at most the previous $H$
    action-observation pairs.
% The infinite sequence can be divided
%     into ``episodes'' of length $H$ where episode $j$ consists of time
%     steps $(j-1)H + 1$ to $jH$.  The penalty at a given time step
%     depends only on the action and observation of the current
%     episode. That is,
%     $\penl(\thetatrue,\datat) = \penl(\thetatrue,\{(\Xj,
%     \YXj)\}_{j=t'}^t)$ and $t'$ is the start time of the current
%     episode.
  \item \label{cond:recoverability} {\bf Recoverability.} There exists
    $\alpha<1$ such that for data sequences $D_1, D_2$ with
    $\penl(\thetatrue, D_1) \leq \penl(\thetatrue, D_2) + \epsilon$,
    we have
 % the following for given $\thetatrue$.  The expectation
 %    $\EEtt{\Yx}$ is over the the observation $\Yx$ at $x$.
% \vspace{-0.05in}
\begin{align*}
\min_{x\in\Xcal} \EEtt{\Yx}[\penl(\thetatrue, D_1\concat (x,\Yx))]
\leq \min_{x\in\Xcal} \EEtt{\Yx}[\penl(\thetatrue, D_2\concat (x,\Yx))] +
\alpha\epsilon.
\end{align*}
The expectation is over the observation $\Yx\sim \PP(y|x,\thetatrue)$.
\item \label{cond:moreisbetter} {\bf More data is better.}  Let
  $D_t,D_{t'}'\in\Dcal$ be data sequences of length $t,t'$ such that
  $D_t \dssubset D_{t'}'$. Then, for every $k \in \NN$, we have
  \begin{align*}
    \EE[\penl(\thetatrue, D_{t'}'\concat D_{t'+1:t'+k})\mid \gopolicy, D_{t'}'] \leq \EE[\penl(\thetatrue, D_{t}\concat D_{t+1:t+k})\mid \gopolicy, D_{t}]
  \end{align*}
  In both expectations, the last $k$ actions are chosen by $\gopolicy$. 
%  Let
%   $H,H'$ be equal length ($|H| = |H'|$) data sequences collected by
%   $\gopolicy$ when it starts from $D$ and $D'$ respectively.
% %   Let $\Ytt{H}\in\Ycal^{|H|}$ be the observations when taking actions in $H$.
% %     Define $\Ytt{H'}$ similarly.
%   Then, for given $\thetatrue$, $\EE[\penl(\thetatrue, D'\concat H')] \leq
%       \EE[\penl(\thetatrue, D\concat H)]$
%   where the expectation is over the future sequences $H,H'$ collected
% by $\gopolicy$.
  \end{condenum}
\end{condition}
Condition~\ref{cond:instantaneous} reduces the problem to an episodic
one, since, when $t$ is a multiple of $H$, it is as if no data has been
collected, corresponding to a reset. As a special case, when $H=1$, we
are in the standard bandit
setting. Condition~\ref{cond:recoverability} states that it is
possible to choose an action from a ``bad" data sequence to improve,
by a multiplicative factor of $\alpha$, in comparison with choosing
the best action from a ``good" data sequence. 
% This implies that even
% if we were unfortunate to have chosen bad actions early on, there is
% opportunity to recover and eventually start performing well, and 
This condition is closely related to diameter/reachability conditions
in infinite horizon
 RL~\cite{jaksch2010near}, which assume that every state (in
particular a good state) is reachable from every other in a small
number of steps. Finally, condition~\ref{cond:moreisbetter} states that
behaving like $\gopolicy$ for $k$ steps from some data sequence yields
lower penalty than when behaving like $\gopolicy$ for $k$ steps from a
prefix. Note that \ref{cond:recoverability}
and~\ref{cond:moreisbetter} involve $\thetatrue$, particular actions
and $\gopolicy$; they suggest that good actions exist, but these good
actions are not known to the decision maker when $\thetatrue$ is not known.

Before stating the main theorem, we introduce the maximum information
gain, $\IGn$, which captures the statistical difficulty of the learning problem.
\begin{align*}
% \IRt &= \frac{\EEtmo[\penlt(\thetatrue,\datat) - \penlt(\thetatrue,\iodatat)]^2}{%
%  \Itmo(\thetatrue; (\Xt, \YXt)) }
% \numberthis\label{eqn:infratio}
%  \\
\IGn &= \max_{\datan\subset\Dcaln} \MI(\thetatrue; \datan).
\numberthis\label{eqn:mig}
\end{align*}
Here $\MI(\cdot;\cdot)$ is the Shannon mutual information, and as such
$\IGn$ measures the maximum information a set of $n$
action-observation pairs can tell us about the true parameter
$\thetatrue$. The quantity appears as a statistical complexity measure
in many Bayesian adaptive data analysis
settings~\cite{srinivas10gpbandits,ma2015active,gotovos2013active}.
% The quantity was first introduced for Gaussian
% processes by~\citet{srinivas10gpbandits}, and has been popularly
% % used in regret bounds in various Bayesian adaptive data analysis settings such as
% used to measure the statistical complexity of various Bayesian adaptive data analysis
% settings such as Bayesian optimisation~\citep{srinivas10gpbandits},
% active learning~\citep{ma2015active},
% and level set estimation~\citep{gotovos2013active}.
Below, we list some examples of common models which demonstrate that
$\IGn$ is typically sublinear in $n$.

\insertprethmspacing
% \begin{example}[Bounds on the MIG for common models~\citep{srinivas10gpbandits}]
\begin{example}We have the following bounds on $\IGn$ for
common models~\cite{srinivas10gpbandits}:
\vspace{-0.02in}
\begin{enumerate}[leftmargin=0.3in]
\setlength\itemsep{-0.02in}
\item \textbf{Finite sets:}
If $\Theta$ is a finite set, $\IGn \leq \log(|\Theta|)$ for all $n$.
\item \textbf{Linear models:}
Let $\Xcal\subset\RR^d$, $\theta\in\RR^d$,
 and $\Yx|x,\theta \sim\Ncal(\theta^\top x, \eta^2)$.
For a multi-variate Gaussian prior on $\thetatrue$, $\IGn\in\bigO(d\log(n))$.
\item \textbf{Gaussian process:}
For a Gaussian process prior with RBF kernel over $\Xcal\subset\RR^d$, and with Gaussian likelihood, we have $\IGn\in\bigO(\log(n)^{d+1})$.
% Let $\Xcal\subset\RR^d$.
% Assume a Gaussian process prior
% $\thetatrue\sim\GP(\mu,\kernel)$,
% and Gaussian observations $\Yx|x,\theta \sim\Ncal(\theta(x), \eta^2)$.
% Here $\mu:\Xcal\rightarrow\RR$ is the prior mean of the GP and
% $\kernel:\Xcal^2\rightarrow\RR$ is the covariance kernel.
% When $\kernel$ is an RBF kernel, $\IGn\in\bigO(\log(n)^{d+1})$.
% % When $\kernel$ is a mat\'ern kernel, $\IGn\in\bigO(n^p\log(n))$ for some $p<1$.
\end{enumerate}
% \vspace{-0.10in}
% \begin{proof}
% For finite $\Theta$, we have $\MI(\thetatrue, \datan) = \Ent(\thetatrue) -
% \Ent(\thetatrue|\datan) \leq \Ent(\thetatrue) \leq \log(|\Theta|)$,
% where $\Ent$ is the shannon Entropy.
% The last two results are taken from~\citet{srinivas10gpbandits}.
% \end{proof}
\end{example}

% In general, one expects the regret in an interactive model to depend on the complexity
% of the model and 
% While the MIG is not known
% The MIG provides us a way to abstract out the statistical complexity of a model.
We now state our main theorem for finite action spaces $\Xcal$
under any one of the above conditions.

\insertprethmspacing
\begin{theorem}
\label{thm:finiteactions} 
Assume \textbf{any one} of conditions~\ref{cond:instantaneous}-\ref{cond:moreisbetter}
hold.
Let $B=H$ under condition~\ref{cond:instantaneous},
$B=1/(1-\alpha)$ under condition~\ref{cond:recoverability},
% $B=1$ under condition~\ref{cond:monotonicity},
and $B=2$ under condition~\ref{cond:moreisbetter}.
Then if $\Xcal$ is finite,
\[
\EE[\lossn(\thetatrue,\pspolicy) - \lossn(\thetatrue, \gopolicy)]
\leq B\sqrt{\frac{n|\Xcal|\IGn}{2}}.
\]
\end{theorem}
% \akshay{No bound on $\penl(\thetatrue,D_n)$ here?}

Theorem~\ref{thm:finiteactions} establishes a sublinear regret bound
for $\pspolicy$ against $\gopolicy$.  The $|\Xcal|$ term captures the
complexity of our action space and $\IGn$ captures the complexity of
the prior on $\thetatrue$.  The $\sqrt{n}$ dependence is in agreement
with prior results for Thompson
sampling~\citep{kaufmann2012thompson,osband2016posterior,russo2016thompson}. Thus,
under any of the above condition, $\pspolicy$ is competitive with the
myopic optimal policy $\gopolicy$, with average regret tending to $0$.
% \akshay{Would be really nice to say something about final regret.}
% In addition, the $\IGn$ term captures the statistical di

To compare with the globally optimal policy $\uopolicy$, we introduce
the notions of \emph{monotonicity} and \emph{adaptive
  submodularity}~\cite{golovin2011adaptive}. 
\insertprethmspacing
\begin{condition}(Monotonicity and Adaptive Submodularity)
\label{cond:submod}
Assume that $\penl$ is monotone, meaning that
for 
$D\in\Dcal$, $x\in\Xcal$,
we have $\EE[\penl(\thetatrue, D\concat\{(x,\Yx)\})] \leq \penl(\thetatrue,D)$.
Assume further that $\penl$ is adaptive submodular, meaning that for all $D \dssubset D'$, $x\in\Xcal$, we have
% Let $D, D'$ be such that $D\dssubset D'$.
% Then, the penalty $\penl$ is monotone, i.e 
% for all $x\in\Xcal$, 
% \EE_{\Yx}[\penl(\thetatrue, D\concat\{(x, \Yx)\})]
% \leq  \penl(\thetatrue, D)$.
% The penalty $\penl$ is adaptive submodular if it satisfies
\begin{align*}
\EE[\penl(\thetatrue, D\concat\{(x,\Yx)\})] - \penl(\thetatrue,D)
\geq 
\EE[\penl(\thetatrue, D'\concat\{(x,\Yx)\})] - \penl(\thetatrue,D').
\end{align*}
\end{condition}
In words, monotonicity states that adding more data reduces the
penalty in expectation, while adaptive submodularity formalises a
notion of diminishing returns. 
That is, performing the same action is more beneficial when we have
less data.
% : the expected benefit in carrying out a
% particular action is larger when we have less data. 
It is easy to see that some assumption is needed here, since even in
simple episodic problems $\gopolicy$ can be arbitrarily worse than
$\uopolicy$.  Under Condition~\ref{cond:submod} it is known that that
$\gopolicy$ closely approximates $\uopolicy$, and using this fact, 
% \toworkon{@ Akshay, 
% Both conditions above are natural in adaptive data collection settings and
% have been used in the adaptive submodularity literature~\citep{gabillon2013adaptive,%
% gotovos2013active}.
% For what follows, we will find it useful to define $\reward = 1-\penl$, which
% can now be interpreted as a reward for collecting a data sequence.
% Recall that by the definition of $\uopolicy$, the data sequence $\uodatan$
% maximises
% $\EE[\reward(\thetatrue, \datan)]$ over all adaptive data choices $\datan$.
we have the following result for $\pspolicy$:
% The result below shows

\insertprethmspacing
\begin{theorem}
\label{thm:submodthm}
Assume that $\penl$ satisfies condition~\ref{cond:submod} and one of
conditions~\ref{cond:instantaneous}-\ref{cond:moreisbetter}.
Let $\mu = 1 - \penl$ and define $B$ as in Theorem~\ref{thm:finiteactions}.
% Let $\datan, \uodatan$ be the data sequences collected by $\pspolicy$ and
% $\uopolicy$ respectively in $n$ experiments.
% Let $B$ be as defined in Theorem~\ref{thm:finiteactions}.
Then, for all $\gamma < 1$, we have %% the following two equivalent statements.
\begin{align*}
% \EE[\penl(\thetatrue, \datan)]
% &\;\leq\; \EE[\penl(\thetatrue,\uodatagn)] +
% \gamma(1 -  \EE[\penl(\thetatrue,\uodatagn)]) + 
% B \sqrt{\frac{|\Xcal|\IGn}{2n}}. \\
\EE[\reward(\thetatrue, \datan) | \datan\sim\pspolicy]
&\;\geq\; (1 - \gamma) \EE[\reward(\thetatrue, \godatatt{\gamma{}n}) | \godatatt{\gamma{}n}\sim\uopolicy] -B \sqrt{\frac{|\Xcal|\IGn}{2n}}.
\end{align*}
\end{theorem}

% When $\reward = 1-\penl$ is viewed as a reward,
The theorem is stated in terms of the final ``reward"
$\reward = 1 - \penl$, which is more natural for submodular
optimisation.
In terms of this reward, the theorem states that
$\pspolicy$ in $n$ steps is guaranteed to perform up to a $1-\gamma$
factor as well as $\uopolicy$ executed for $\gamma n < n$ steps, up to
an additive $\sqrt{\IGn/n}$ term. The result captures both
approximation and estimation errors, in the sense that we are using a
myopic policy to approximate a globally optimal one, and we are
learning a good myopic policy from data. In comparison, prior works on
adaptive submodular optimisation focus on approximation errors and
typically achieve $1-1/e$ approximation ratios against the $n$ steps
of $\uopolicy$. Our bound is quantitatively worse, but focusing on a
much more difficult task, and we view the results as complementary.
% The latter term vanishes as $n\rightarrow\infty$ so, asymptotically $\pspolicy$ achieves
% this constant factor.
% This result is analogous to well known results in submodular \emph{maximisation}
% where greedy policies are known to achieve up to a constant factor of the 
% global maximum.
% However, while adaptive submodular maximisation is known to achieve a
% $1-1/e$ approximation to the global optimum executed for $n$ steps,
% $\pspolicy$ can only achieve this for $\gamma n$ steps, and moreover,
% the constant factor decreases with $\gamma$.
% Hence, the result in Theorem~\ref{thm:submodthm} is weaker than submodular maximisation
% in two aspects.
% We believe that this is primarily due to the fact that the rewards are not known or
% observed which makes our settings considerably more challenging
% than classical submodular maximisation.
We finally note that an analogous bound holds against $\gopolicy$,
since it is necessarily worse that $\uopolicy$.
%  statement akin to Theorem~\ref{thm:submodthm} can also
% be stated with respect to the myopic policy $\gopolicy$ since it is necessarily
% worse than $\uopolicy$.

\akshay{I suggest we state one of these formally. Probably the
  infinite action case is most relevant.} 
Finally, we mention that
the above results can be generalised to very large or infinite action
spaces under additional structure on the
problem using known techniques~\citep{russo2016thompson,bubeck2011x};
this is tangential to the goal of this paper.
Algorithm~\ref{alg:bdoe} can be applied as is in either synchronously
or asychronously parallel settings with $m$ workers. By following the
analysis for parallel Thompson sampling for Bayesian
optimisation~\citep{kandasamy2018parallel}, one can
obtain results similar to Theorems~\ref{thm:finiteactions}
and~\ref{thm:submodthm} with mild dependence on $m$.

\section{Conclusion}
\label{sec:conclusion}
This paper studies myopic algorithms for sequential design of
experiments in a Bayesian setting. Our formulation is quite general,
allowing practitioners to incorporate domain knowledge via a
probabilistic model, and specify design goals via a penalty
function that may depend on system characteristics.
We also exploit advances in probabilistic programming for
further generality and ease of use.
Our
empirical results demonstrate that our general formulation has broad
applicability. 
Our algorithm performs favourably in comparison with
more specialised methods, and more importantly, enables complex DOE tasks where
existing methods are not applicable.
Our theoretical results establish
conditions under which a myopic algorithm based on posterior sampling
is competitive with myopic and globally optimal policies, both of
which know the underlying system parameters.
A natural theoretical question for future work is to study policies
with $k$-step lookahead, interpolating between myopic policies and
fully optimal ones.

% On the practical side, we also plan to export our
% implementation as a software package, which we hope will promote
% adoption and facilitate further scientific and industrial
% applications.

% \vfill
% \newpage

\subsection*{Acknowledgements}
\vspace{-0.1in}
This research is partly funded by DOE grant DESC0011114, NSF grant
IIS1563887, the Darpa D3M program, AFRL, and
Toyota Research Institute, Accelerated Materials Design \& Discovery (AMDD)
program.
KK is supported by a Facebook fellowship and a Siebel scholarship.
\vspace{-0.1in}

{\small
\renewcommand{\bibsection}{\section*{References\vspace{-0.1em}} }
\setlength{\bibsep}{1.1pt}
\bibliography{kky,bib_cox}
}
\bibliographystyle{plainnat}

% \newpage

\appendix

\section{Some Ancillary Material}
\label{sec:ancillary}
\label{app:ancillary}

We will need the following technical results for our analysis.
The first is a version of Pinsker's inequality.

\insertprethmspacing
\insertprethmspacing
\begin{lemma}[Pinsker's inequality]
\label{lem:pinsker}
Let $X,Z\in\Xcal$ be random quantities and $f:\Xcal\rightarrow[0,B]$.
Then, \emph{$\big|\EE[f(X)] - \EE[f(Z)]\big| \leq B \sqrt{\frac{1}{2}\KL(P(X)\|P(Z))}$}.
\end{lemma}

The next, taken from~\citet{russo2016thompson},
 relates the KL divergence to the mutual information for two random quantities
$X, Y$.

\insertprethmspacing
\insertprethmspacing
\begin{lemma}[\citet{russo2016thompson}, Fact 6]
\label{lem:klmi}
For random quantities $X,Z\in\Xcal$, \\
\emph{
$I(X;Z) = \EE_X[\KL(P(Y|X)\|P(Y))]$.
}
\end{lemma}

The next result is a property of the Shannon mutual information.

\insertprethmspacing
\begin{lemma}
\label{lem:misuperset}
Let $X, Y, Z$ be random quantities such that $Y$ is a deterministic function of $X$.
Then, $I(Y;Z) \leq I(X;Z)$.
\end{lemma}
\insertpostthmspacing
\begin{proof}
Let $Y'$ capture the remaining randomness in $X$ so that $X=Y\cup Y'$.
Then, since conditioning reduces entropy,
$I(Y;Z) = H(Z) - H(Z|Y) \leq H(Z) - H(Z|Y\cup Y') = I(X;Z)$.
\end{proof}

\section{Proofs}

\subsection{Notation and Set up}

In this subsection, we will introduce some notation, prove some basic lemmas,
and in general, lay the groundwork for our analysis.
$\PP,\EE$ denote probabilities and expectations.
$\PPt,\EEt$ denote probabilities and expectations when conditioned on
the actions and observations up to and including
time $t$, e.g. for any event $E$, $\PPt(E) = \PP(E|\datat)$.
For two data sequences $A,B$, $A\concat B$ denotes the concatenation of the two
sequences.
When $x\in\Xcal$, $\Yx$ will denote the random observation from
$\PP(Y|x,\theta)$.

% $\pthetat$ denotes the posterior for $\thetatrue$ when conditioned on $\datat$,
% $\pthetat(\cdot) = \PP(\thetatrue\in\cdot|\datat)$.
% 
% We will introduce some notation before we proceed.
Let $\datat\in\Dcalt$ be a data sequence of length $t$.
Then, $\Qpi(\datat,x,y)$ will denote the expected total penalty when we take action
$x\in\Xcal$, observe $y\in\Ycal$ and then execute policy $\policy$ for the remaining
$n-t-1$ steps.
That is,
\begin{align*}
\numberthis\label{eqn:Qdefn}
&\Qpi(\datat, x, y) = \EE\big[\sumpenl(\thetatrue, \datat\concat\{(x, y)\}\concat\Ftptn) \big] \\
&\hspace{0.05in}
= \sum_{j=1}^t\penl(\thetatrue, \dataj) +
  \penl(\thetatrue, \dataj\concat\{(x,y)\}) +
  \EE_{\Ftptn}\bigg[
  \sum_{j=t+2}^n \penl(\thetatrue, \dataj\concat\{(x,y)\}\concat
      \Fttnn{t+2}{j} ) \bigg].
\end{align*}
Here, the action-observation pairs collected by $\policy$
from steps $t+2$ to $n$ are $\Ftptn$.
The expectation is over the observations and any randomness in $\policy$.
While we have omitted for conciseness,
$\Qpi$ is a function of the true parameter $\thetatrue$.
Let $\dttpp{t}{\pi}$ denote the distribution of $\datat$ when following a policy $\pi$
for the first $t$ steps. % to be $\dttpp{t}{\pi}$.
We then have,
\begin{align*}
J(\thetatrue,\policy) = \EE_{\datat\sim\dttpp{t}{\policy}}
\big[
\EE_{X\sim\policy(\datat)}[
%      \EE_{\YX\sim\PP(Y|X,\thetatrue)}[
  \Qpp{\policy}(\datat, X, Y)]\big],
\numberthis \label{eqn:lossQ}
\end{align*}
where, recall, $\YX$ is drawn from $\PP(Y|X,\thetatrue)$.
The following Lemma decomposes the regret
$\lossn(\thetatrue,\policy) - \lossn(\thetatrue,\gopolicy)$
% ~\eqref{eqn:regretDefn}
as a sum of terms which are convenient to analyse.
The proof is adapted from Lemma 4.3 in~\citet{ross2014reinforcement}.

\insertprethmspacing
\begin{lemma}
For any two policies $\policyone,\policytwo$,
\begin{align*}
&J(\thetatrue,\policyone) - J(\thetatrue,\policytwo) = \\
% = \sum_{t=1}^n \EE_{\datatmo \sim \dtmopo, X\sim\policyone(\datatmo)}\left[
%  \Qpt(\datatmo, X) \right]
%  -\EE_{\datatmo \sim \dtmopo, X\sim\policytwo(\datatmo)}\left[
%  \Qpt(\datatmo, X) \right]
&\hspace{0.2in} \sum_{t=1}^n \EE_{\datatmo \sim \dtmopo}\left[
  \EE_{X\sim\policyone(\datatmo)}\left[
 \Qpt(\datatmo, X, \YX) \right]
 -\EE_{X\sim\policytwo(\datatmo)}\left[
 \Qpt(\datatmo, X, \YX) \right]
  \right]
\end{align*}
\label{lem:valuedecomp}
\begin{proof}
Let $\policy^t$ be the policy that follows $\policyone$ from time step $1$ to $t$, and
then executes policy $\policytwo$ from $t+1$ to $n$.
Hence, by~\eqref{eqn:lossQ}, 
\begin{align*}
J(\thetatrue, \policy^t) &=
\; \EE_{\datatmo\sim\dttpp{t-1}{\policyone}}
\big[\EE_{X\sim\policyone(\datatmo)}[\Qpp{\policytwo}(\datatmo, X, \YX)]\big] \\
&=\;
\EE_{\datat\sim\dttpp{t}{\policyone}}
\big[\EE_{X\sim\policytwo(\datat)}[\Qpp{\policytwo}(\datat, X, \YX)]\big].
\end{align*}
The claim follows from the observation,
$
J(\thetatrue, \policyone) - J(\thetatrue, \policytwo) =
J(\thetatrue, \policy^n) - J(\thetatrue, \policy^0) =
\sum_{t=1}^n J(\thetatrue, \policy^t) - J(\thetatrue, \policy^{t-1}).
$
% &\hspace{0.2in}
% = \sum_{t=1}^n J(\thetatrue, \policy^t) - J(\thetatrue, \policy^{t-1}) \\
\end{proof}
\end{lemma}

We will use Lemma~\ref{lem:valuedecomp} with $\policytwo$ as the policy $\gopolicy$
which knows $\thetatrue$ and with $\policyone$ as the policy $\policy$
whose regret we wish to bound.
For this, 
denote the action chosen by $\policy$ when it has seen data $\datatmo$
as $\Xt$ and that taken by $\gopolicy$ as $\Xpt$. %, i.e. $\Xt\sim\policy(\datatmo)$
% For this, 
% denote the action chosen by $\policy$ when it has seen data $\datatmo$
% as $\Xt$.
% Denote the action-observations pairs when choosing
% $\Xt\sim\policy(\datatmo)$ and then following $\gopolicy$
% as $\Htn$, i.e. $\Htn$ includes $(\Xt, \YXt)$ and has length $n-t+1$.
% Similarly,  denote the action chosen by $\gopolicy$ when it has seen data $\datatmo$
% as $\Xpt$ and the ensuing $n-t+1$ action-observations pairs when following
% $\gopolicy$ as $\Hptmotn$.
By Lemma~\ref{lem:valuedecomp} and equation~\eqref{eqn:Qdefn} we have,
\begin{align*}
% \EE_{\thetatrue}[\loss(\thetatrue,\policy) - \loss(\thetatrue,\gopolicy)]
% &= \sum_{t=1}^n \EEtmo\Big[
%  \underbrace{\Qpp{\gopolicy}(\datatmo,\Xt,\YXt)
%  - \Qpp{\gopolicy}(\datatmo,\Xpt,\YXpt)}_{r_t} \Big] \\
% % \underbrace{\Qpp{\gopolicy}(\datatmo,\Xt)
% %  - \Qpp{\gopolicy}(\datatmo,\Xpt) }_{r_t}  \Big]
% &= \sum_{t=1}^n \EEtmo\Big[
% \underbrace{\sumpenl(\datatmo\concat\Htn)
%  - \sumpenl(\datatmo\concat\Hptmotn)}_{r_t} \Big]
\EE_{\thetatrue}[\loss(\thetatrue,\policy) - \loss(\thetatrue,\gopolicy)]
&= \sum_{t=1}^n \EEtmo\Big[
 \Qpp{\gopolicy}(\datatmo,\Xt,\YXt)
 - \Qpp{\gopolicy}(\datatmo,\Xpt,\YXpt) \Big]. %\\
% \underbrace{\Qpp{\gopolicy}(\datatmo,\Xt)
%  - \Qpp{\gopolicy}(\datatmo,\Xpt) }_{r_t}  \Big]
% &= \sum_{t=1}^n \EEtmo\Big[
% \underbrace{\sumpenl(\datatmo\concat\Htn)
%  - \sumpenl(\datatmo\concat\Hptmotn)}_{r_t} \Big]
\end{align*}
Note that $\EEtmo$, which conditions on the data sequence
$\datatmo$ collected by $\policy$,
encompasses three sources of randomness.
The first is due to the randomness in the problem due to $\thetatrue\sim\pthetazero$,
the second is due to the observations $Y\sim\PP(\cdot|X,\thetatrue)$,
and the third 
is an external source of randomness $U$ used by the decision maker in
step~\ref{step:sampletheta} in Algorithm~\ref{alg:bdoe}.
While the actions $\{\Xt\}_t$ chosen depends on all sources of randomness, these
three sources are themselves independent.
For example, we can write $\EEtmo[\cdot] = \EEU[\EEtmorho[\EEtmoY[\cdot]]]$
where $\EEU$ captures the randomness by the decision maker,
$\EEtmorho$ due to the prior
and
$\EEtmoY$ due to the observations.
With this in consideration, 
define
\begin{align*}
\qt(x,y) = \EEtt{Y,t+1:n}[\Qpp{\gopolicy}(\datatmo,x,y)],
\numberthis
\label{eqn:qtdefn}
\end{align*}
where $\EEtt{Y,t+1:n}$ is the expectation over the observations from time step
$t+1$ to $n$.
$\qt$ is the expected total penalty when we
have data $\datatmo$ collected by $\policy$, then execute action $x$ at time $t$,
observe $y$ and then follow $\gopolicy$ for the remaining time steps.
Note that $\qt$ is a deterministic function of $\thetatrue$, $x$, and $y$
since $\gopolicy$ is a deterministic policy and the randomness of future observations
has been integrated out.
We can now write,
\begin{align*}
\EE[\loss(\thetatrue,\policy) - \loss(\thetatrue,\gopolicy)]
&= \sum_{t=1}^n \EEtmo\Big[
  \qt(\Xt, \YXt) - \qt(\Xpt, \YXpt)
  \Big],
\numberthis
\label{eqn:lossqt}
\end{align*}
where $\EEtmo$ inside the summation is over the randomness in $\thetatrue$,
the randomness of the policy in choosing $\Xt$ and
the observations $\YXt, \YXpt$.

\subsection{Proof of Theorem~\ref{thm:finiteactions}}

We will let $\Ptildetmo$ denote the distribution of $\Xt$ given $\datatmo$;
i.e. $\Ptildetmo(\cdot) = \PPtmo(\Xt=\cdot)$.
The density (Radon-Nikodym derivative) $\ptildetmo$ of $\Ptildetmo$ can be expressed as
$\ptildetmo(x) = \int_\Theta p_{\star}(x|\theta)p(\theta|\datatmo)\ud \theta$
where $p_\star(x|\theta)$ is the density of the maximiser given $\theta$
and $p(\theta|\datatmo)$ is the posterior density of $\theta$ conditoned
on $\datatmo$.
Hence, $\Xt$ has the same distribution as $\Xpt$;
% i.e. $\Xt \stackrel{d}{=} \Xpt$.
i.e. $\PPtmo(\Xpt=\cdot) = \Ptildetmo(\cdot)$. % \stackrel{d}{=} \Xpt$.
This will form a key intuition in our analysis.
To this end, 
we begin with a technical result, whose proof is adapted from~\citet{russo2016thompson}.
We will denote by $\MItmo(A;B)$ the mutual information between two variables
$A,B$ under the posterior measure after having seen $\datatmo$;
i.e. $\MItmo(A;B) = \KL(\PPtmo(A,B)\|\PPtmo(A)\cdot\PPtmo(B))$.

\insertprethmspacing
\begin{lemma}
\label{lem:russolemmas}
Assume that we have collected a data sequence $\datatmo$.
Let the action taken by $\pspolicy$ at time instant $t$ with $\datatmo$ be $\Xt$ and the
action taken by $\gopolicy$ be $\Xpt$.
% Let $\ptildetmo$ be the density of $\Xt$ given $\datatmo$.
Then,
\begin{align*}
&
\EEtmo[\qt(\Xt,\YXt) -  \qt(\Xpt,\YXpt)]
=\sum_{x\in\Xcal} 
\big(\EEtmo[\qt(x,\Yx)] - \EEtmo[\qt(x,\Yx)|\Xpt=x]\big) \Ptildetmo(x) \\
&\MItmo(\Xpt; (\Xt, \YXt)) = \sum_{x_1,x_2\in\Xcal}
  \KL(\PPtmo(\Yxone|\Xpt=x_2)\| \PPtmo(\Yxone) )\,
  \Ptildetmo(x_1)\Ptildetmo(x_2)
\end{align*}
\begin{proof}
The proof for both results uses the fact that
$\PPtmo(\Xt=x) = \PPtmo(\Xpt=x) = \Ptildetmo(x)$.
For the first result,
\begin{align*}
&\EEtmo[\qt(\Xt,\YXt) -  \qt(\Xpt,\YXpt)] \\
&\hspace{0.02in} =
\sum_{x\in\Xcal}
\PPtmo(\Xt=x)\EEtmo[\qt(\Xt,\YXt)|\Xt=x]
-
\sum_{x\in\Xcal}
\PPtmo(\Xpt=x)\EEtmo[\qt(\Xpt,\YXpt)|\Xpt=x] \\
&\hspace{0.02in} =
\sum_{x\in\Xcal}
\PPtmo(\Xt=x)\EEtmo[\qt(x,\Yx)]
- \sum_{x\in\Xcal} \ud\PPtmo(\Xpt=x)\EEtmo[\qt(x,\Yx)|\Xpt=x] \\
&\hspace{0.02in} =
\sum_{x\in\Xcal}
\big(\EEtmo[\qt(x,\Yx)] - \EEtmo[\qt(x,\Yx)|\Xpt=x]\big) \Ptildetmo(x)  \,.
\end{align*}
The second step uses the fact that the observation $\Yx$ does not depend on
the fact that $x$ may have been chosen by $\pspolicy$;
this is because $\pspolicy$ makes its decisions based on past data $\datatmo$
and is independent of $\thetatrue$ given $\datatmo$.
$\Yx$ however can depend on the fact that $x$ may have been the action chosen by
$\gopolicy$ which knows $\thetatrue$.
For the second result,
\begin{align*}
&\MItmo(\Xpt; (\Xt, \YXt)) =
  \MItmo(\Xpt; \Xt) + \MItmo(\Xpt; \YXt|\Xt)
= \MItmo(\Xpt; \YXt|\Xt)
\\ &\hspace{0.1in} =
\sum_{x_1\in\Xcal} \PPtmo(\Xt=x_1)\,\MItmo(\Xt;\YXt|\Xt=x)
=
\sum_{x_1\in\Xcal} \Ptildetmo(x_1)\ud(x)\,\MItmo(\Xpt;\Yxone)
\\ &\hspace{0.1in} =
\sum_{x_1\in\Xcal} \Ptildetmo(x_1)\ud(x)
\sum_{x_2\in\Xcal} \PPtmo(\Xpt=x_2)
\,\KL(\PP(\Yxone|\Xpt=x_2)\|\PP(\Yxone))
\\ &\hspace{0.1in} =
  \sum_{x_1,x_2\in\Xcal}
  \KL(\PPtmo(\Yxone|\Xpt=x_2)\| \PPtmo(\Yxone))\,
  \ptildetmo(x_1)\ptildetmo(x_2)
\end{align*}
The first step uses the chain rule for mutual information.
The second step uses that $\Xt$ is chosen based on an external source of randomness and
$\datatmo$;
therefore, it
is independent of $\thetatrue$ and hence $\Xpt$ given $\datatmo$.
The fourth step uses that $\Yxone$ is independent of $\Xt$.
The fifth step uses lemma~\ref{lem:klmi} in Appendix~\ref{app:ancillary}.
\end{proof}
\end{lemma}

The next Lemma uses the conditions on $\penl$ given in
Condition~\ref{cond:structural_cond}
to show that $\qt$~\eqref{eqn:qtdefn} is bounded.
This essentially establishes that the effect of a single bad
action is bounded on the long run penalties.

\insertprethmspacing
\begin{lemma}
\label{lem:Bbound}
% Let any one of conditions~\ref{cond:instantaneous}-\ref{cond:moreisbetter}
% be true, and $B$ be as given in Theorem~\ref{thm:finiteactions}.
Let $B$ be as given in Theorem~\ref{thm:finiteactions} for any of
conditions~\ref{cond:instantaneous}-\ref{cond:moreisbetter}.
Then, $\;\sup \qt - \inf \qt \leq B$.
% Then,\\ $\sup \qt - \inf \qt \leq B$.
\begin{proof}
In this proof, $(x,y), (x',y')\in\Xcal\times\Ycal$ will be two pairs of
action-observations.
Denote the action-observations pairs when following $\gopolicy$ after $(x,y)$
by $\Htn$, i.e. $\Htn$ starts with $(x, y)$ and has length $n-t+1$.
Similarly define $\Hptn$ for $(x',y')$. Then
$\qt(x,y) - \qt(x',y') = \EEtt{Y,t+1:n}[r_t]$, where,
\begin{align*}
r_t &= \sumpenl(\thetatrue, \datatmo\concat\Htn)
                  - \sumpenl(\thetatrue, \datatmo\concat\Hptn)
\\
&= \sum_{j=t}^n \Big(\penl(\thetatrue, \datatmo\concat\Htj)
                  - \penl(\thetatrue, \datatmo\concat\Hptj) \Big)
\label{eqn:rtdefn} \numberthis
\end{align*}
% Here $\sumpenl$ is as defined in~\eqref{eqn:sumpenldefn}.
We will now prove $\qt(x,y) - \qt(x',y') \leq B$ separately for each condition.

\textbf{Condition~\ref{cond:instantaneous}:}
Under this setting,  $\gopolicy$ which knows $\thetatrue$,
will behave identically after the end of the current episode.
This is  because the penalty
at the next episode will not depend on the data collected during the current episode.
Therefore, the summation in~\eqref{eqn:rtdefn} is from $t$ to the end
of the episode $s$.
We hence have, $r_t \leq s - t \leq H = B$.

\textbf{Condition~\ref{cond:recoverability}:}
Denote $\epsilon_j = \penl(\thetatrue, \datatmo\concat\Htj)
- \penl(\thetatrue, \datatmo\concat\Hptj)$.
$\epsilon_j$ is a random quantity for $j\geq t+1$ as it depends on the observations.
We have $r_t = \sum_{j=t}^n \epsilon_j$.
Observe that, at time step $j$, $\gopolicy$ chooses the action to maximise
$\EEtt{\Yx}[\penl(\thetatrue,\datatmo\concat\Htj\concat\{(x,\Yx)\}]$
when starting with $(x,y)$, and 
$\EEtt{\Yx}[\penl(\thetatrue,\datatmo\concat\Hptj\concat\{(x,\Yx)\}]$
when starting with $(x',y')$.
Hence, condition~\ref{cond:recoverability} implies that,
$\EE_{\Ytt{j+1}}[\epsilon_{j+1}] \leq \alpha \epsilon_j$.
% An inductive argument leads to the concl
An inductive argument leads us to,
$\EE_{\Ytt{t+1},\dots,\Ytt{j}}[\epsilon_j] \leq \alpha^{j-t}\epsilon_t$.
However, since $\penl$ maps to $[0,1]$, $\epsilon_t \leq 1$.
Hence $\EEtt{Y,t+1:n}[r_t] \leq \sum_j\alpha^{j-t}\epsilon_t \leq 1/(1-\alpha)=B$.

\textbf{Condition~\ref{cond:moreisbetter}:}
For the purposes of this analysis, we will allow a decision maker to take
no action at time $t-1$ and denote this by $\emptyset$,
i.e. $\datatmo\concat\emptyset\concat\Ftpon$ means the action observation pairs were
$\datatmo$ from time $1$ to $t-1$, then there was no action at time $t$ and then
from time $t+1$ to $n$, the action observation pairs were $\Ftpon$.
In doing so, the decision maker incurs a penalty of $\penl(\thetatrue,\datatmo)$
at step $t+1$.
% In doing so, the decision maker incurs a (maximum) penalty of $1$ at step $t+1$.
% \toworkon{the penalty can also be $\penl(\thetatrue,\datat)$}.
% \toworkon{the penalty can also be $1$}.
Correspondingly, we have,
\begin{align*}
\sumpenl(\thetatrue,\datatmo\concat\emptyset\concat\Ftptn)
 = \sum_{j=1}^{t-1}\penl(\thetatrue,\dataj) + \penl(\thetatrue,\datatmo)
  + \sum_{j=t+2}^n\penl(\thetatrue,\datatmo\concat\Fttnn{t+1}{j}).
\end{align*}
Adding and subtracting $
\sumpenl(\thetatrue, \datatmo\concat\emptyset\concat\Hptnmo)\big)
$ to $r_t$ we have,
% Writing
\begin{align*}
r_t =\;&
\big(\sumpenl(\thetatrue, \datatmo\concat\Htn) -
\sumpenl(\thetatrue, \datatmo\concat\emptyset\concat\Hptnmo)\big)
\numberthis\label{eqn:rtsumtwo}
\;+ \\
&
\big(
\sumpenl(\thetatrue, \datatmo\concat\emptyset\concat\Hptnmo) -
\sumpenl(\thetatrue, \datatmo\concat\Hptn)
\big).
\end{align*}

The second term above can be bounded by $1$ since $\lambda$ maps to $[0,1]$.
\begin{align*}
&\sumpenl(\thetatrue, \datatmo\concat\emptyset\concat\Hptnmo) - 
\sumpenl(\thetatrue, \datatmo\concat\Hptn)
\numberthis\label{eqn:rtone}
\\
&\hspace{0.3in} =
\penl(\thetatrue,\datatmo) + \sum_{j=t+1}^{n-1}\penl(\thetatrue,\datatmo\concat\Hptj)
                           - \sum_{j=t+1}^{n}\penl(\thetatrue,\datatmo\concat\Hptj) \\
&\hspace{0.3in} =
 \penl(\thetatrue,\datatmo) - \penl(\thetatrue, \datatmo\concat\Hptn) \leq 1.
\end{align*}
For the first term, we have,
\begin{align*}
% \EEtmoY\big[
&\big(\sumpenl(\thetatrue, \datatmo\concat\Htn) -
\sumpenl(\datatmo\concat\emptyset\concat\Hptnmo)\big)
\numberthis\label{eqn:rttwo}
\\
&\hspace{0.05in}=
  \big(\penl(\thetatrue, \datatmo\concat\{(x,y)\}) - \penl(\thetatrue, \datatmo)\big)
  \;+\;
% \\
% &\hspace{0.5in}
  \sum_{j=t+1}^n\big(\penl(\thetatrue, \datat\concat\Htj) -
                \penl(\thetatrue, \datat\concat\Hptjmo)\big).
\end{align*}
Recall that
the actions in $\Htj$ are chosen to maximise the expected future rewards.
By condition~\ref{cond:moreisbetter} and since
$\datatmo\cup\{(x,y)\}\supset\datatmo$, each of the 
$n-t$ terms in the RHS summation is less than or equal to zero in expectation
over the observations.
Since $
\penl(\thetatrue, \datatmo\concat\{(x,y)\}) -  \penl(\thetatrue, \datatmo) \leq 1$,
the above term is at most $1$ in expectation over $\EEtt{Y,t+1:n}$.
Combining this with~\eqref{eqn:rtone} gives us
$\EEtt{Y,t+1:n}[r_t] \leq 2 = B$.
% Hence, the above therm is at
\end{proof}
\end{lemma}

We are now ready to prove theorem~\ref{thm:finiteactions}.

\textbf{\emph{Proof of Theorem~\ref{thm:finiteactions}:}}
Using the first result of Lemma~\ref{lem:russolemmas}, we have,
\begingroup
\allowdisplaybreaks
\begin{align*}
&\EEtmo[\qt(\Xt,\YXt) - \qt(\Xpt,\YXpt)]^2
\\ &\hspace{0.2in}=
\bigg(\sum_{x\in\Xcal}\Ptildetmo(x)
\big(\EEtmo[\qt(x,\Yx)] - \EEtmo[\qt(x,\Yx)|\Xpt=x]\big)
\bigg)^2
\\ &\hspace{0.2in}\leq
|\Xcal| \sum_{x\in\Xcal}
\Ptildetmo(x)^2\big(\EEtmo[\qt(x,\Yx)] - \EEtmo[\qt(x,\Yx)|\Xpt=x]\big)^2
\\ &\hspace{0.2in}\leq
|\Xcal| \sum_{x_1,x_2\in\Xcal}
\Ptildetmo(x_1)\Ptildetmo(x_2)
\big(\EEtmo[\qt(x_1,\Yxone)] - \EEtmo[\qt(x_1,\Yxone)|\Xpt=x_2]\big)^2
\\ &\hspace{0.2in}\leq
\frac{|\Xcal|B^2}{2} \sum_{x_1,x_2\in\Xcal}
\Ptildetmo(x_1)\Ptildetmo(x_2)
\KL(\PPtmo(\Yxone|\Xpt=x_2)\| \PPtmo(x_1))
\\ &\hspace{0.2in}=\frac{1}{2}
|\Xcal|B^2 \MItmo(\Xpt;(\Xt, \YXt))
\,\leq\,
\frac{1}{2}
|\Xcal|B^2 \MItmo(\thetatrue;(\Xt, \YXt))
\end{align*}
\endgroup
Here, the second step uses the Cauchy-Schwarz inequality and
the third step uses the fact that the previous line can be viewed as the
diagonal terms in a sum over $x_1,x_2$.
The fourth step uses a version of Pinsker's inequality given in
Lemma~\ref{lem:pinsker} of Appendix~\ref{app:ancillary}
and the fifth step uses the second result of Lemma~\ref{lem:russolemmas}.
The last step uses Lemma~\ref{lem:misuperset} and the fact that
$\Xpt$ is a deterministic function of $\thetatrue$ given $\datatmo$.
Now, using~\eqref{eqn:lossqt} and the Cauchy-Schwarz inequality we have,
\begin{align*}
\EE[\loss(\thetatrue,\pspolicy) - \loss(\thetatrue,\gopolicy)]^2
&\leq n \sum_{t=1}^n
\frac{1}{2}
|\Xcal|B^2 \MItmo(\thetatrue;(\Xt, \YXt))
=
\frac{1}{2}
|\Xcal|B^2 \MI(\thetatrue;\datan)
\end{align*}
Here the last step uses the chain rule of mutual information in the following form,
\begin{align*}
\sum_t\MItmo(\thetatrue;(\Xt,\YXt)) 
= \sum_t\MI(\thetatrue;(\Xt,\YXt)|\{(\Xj,\YXj)\}_{j=1}^{t-1}) 
= \MI(\thetatrue; \{(\Xj,\YXj)\}_{j=1}^{n}).
\end{align*}
The claim follows from the observation,
$\MI(\thetatrue; \datan)\leq \IGn$.
\hfill
\qedsymbol

\subsection{Proof of Theorem~\ref{thm:submodthm}}

Let $D$ be the data sequence collected by a policy $\policy$.
For brevity, let
$\penlbar(D) = \EE[\penl(\thetatrue, D)]$ denote the expected penalty
when executing policy $\policy$ for given $\thetatrue$.
Note that $\penlbar$ is a function of the policy $\policy$.
For example,
if $D_1, D_2$ was collected by policies $\policyone,\policytwo$,
$\penlbar(D_1\concat D_2)$ will denote the expected penalty of
the policy which executes $\policyone$ for $|D_1|$ steps, then
starts executing $\policytwo$ without considering the data collected by $\policyone$.
% Further, let
We begin with the following Lemma which shows that $\gopolicy$ performs as well as
the globally optimal adaptive policy $\uopolicy$ up to a constant factor.
Note that both $\gopolicy$ and $\uopolicy$ know $\thetatrue$.

\insertprethmspacing
\begin{lemma}
\label{lem:gouobound}
Let $\uodatam$ be the data collected $\uopolicy$ in $m$ steps
and $\godatan$ be the data collected by $\gopolicy$ in $n$ steps.
Then, under condition~\ref{cond:submod},
\[
\penlbar(\godatan) \leq \penlbar(\uodatam) + 
e^{-\frac{n}{m}} \left( 1- \penlbar(\uodatam)\right)
\]
\begin{proof}
The proof follows the analysis of of similar myopic algorithms under
submodularity assumptions~\citep{goundan2007revisiting,nemhauser1978analysis}.
We begin with the following calculations for $t< n$.
For $t\leq n$, let $\godatat$ be the first $t$ points collected by $\gopolicy$,
For $j\leq m$, let $\uodataj$ be the first $j$ points collected by $\uopolicy$,
and $\uoXj$ be the $j$\ssth point collected by $\uopolicy$.
We have,
\begin{align*}
&\penlbar(\uodatam) \geq  \penlbar(\godatat\concat\uodatam)
=\penlbar(\godatat) +
  \sum_{j=1}^m \Big(\penlbar(\godatat \concat \uodataj) -
  \penlbar(\godatat \concat \uodatajmo)\Big)
\label{eqn:penlbarbound}\numberthis
\\
&\hspace{0.1in}
\geq\penlbar(\godatat) + \sum_{j=1}^m 
\Big(\penlbar(\godatat \concat \{(\uoXj,\YuoXj)\})
  - \penlbar(\godatat)\Big)
\geq\penlbar(\godatat) + \sum_{j=1}^m \Big(\penlbar(\godatatpo) - \penlbar(\godatat) \Big)
\end{align*}
Here, the first step uses monotonicity on $\penl$ (condition~\ref{cond:submod}),
and the second step
is a telescoping sum.
The third step uses the diminishing returns property in condition~\ref{cond:submod}
noting that $\godatat \concat \uodatajmo \supset \godatat$;
note that $\uoXj$ is the last element of $\uodataj$.
The last step uses that $\gopolicy$ chooses the best next action in expectation, and that
it knows $\thetatrue$, 
and hence $\EE[\penl(\thetatrue, \godatatpo)|\datat]
\leq \EE[\penl(\thetatrue, \godatat\cup\{(x,\Yx)\})|\datat]$
for all actions $x\in\Xcal$.

Now let $\delta_t = \penlbar(\uodatam) - \penlbar(\godatat)$.
Equation~\eqref{eqn:penlbarbound} takes the form,
$-\delta_t \geq m(\delta_{t+1} - \delta_t)$
which implies $\delta_{t+1}\leq (1-1/m)\delta_t$.
Applying this recursively 
to obtain 
$\delta_n \leq \big(1-1/m\big)^n \delta_0 \leq e^{-n/m}\delta_0$
and observing that $\delta_0 = \penl(\emptyset)
-\penl(\uodatam) \leq 1 - \penl(\uodatam)$ yields the result.
% \[ 
% \delta_n \leq \bigg(\frac{m}{m+1}\bigg)^n \delta_0
% \;\implies\;
% \penlbar(\godatan) \leq \penlbar(\uodatam) + 
% \bigg(\frac{m}{m+1}\bigg)^n \left( 1- \penlbar(\uodatam)\right)
% \]
\end{proof}
\end{lemma}

\begin{proof}[\textbf{Proof of Theorem~\ref{thm:submodthm}}]
Let $\datan$ be the data collected by $\pspolicy$.
By monotonicity of $\penlbar$, and the fact that the minimum is smaller than the average
we have $\penlbar(\datan) \leq \frac{1}{n}\sum_{t=1}^n \penlbar(\datat)$.
Hence, %by Theorem~\ref{thm:finiteactions} we have,
\begin{align*}
\EE[\penl(\thetatrue, \datan)] &\leq
\EE\left[ \frac{1}{n} \sum_{t=1}^n \penlbar(\godatat) \right]
 + B\sqrt{\frac{|\Xcal|\IGn}{2n}} \\
%  &\leq 
% \EE\left[ \frac{1}{n} \sum_{t=1}^n \penlbar(\godatat) \right]
%  + B\sqrt{\frac{|\Xcal|\IGn}{2n}} \\
&\leq
\EE\left[ \penlbar(\uodatam) + 
\frac{1}{n}\sum_{t=1}^n e^{-t/m} \big(1 - \penlbar(\uodatam) \big)
\right]
 + B\sqrt{\frac{|\Xcal|\IGn}{2n}} \\
&\leq
\EE[\penl(\thetatrue, \uodatam)] + 
(1 - \EE[\penl(\thetatrue, \uodatam)])
\frac{1}{n}\sum_{t=1}^n e^{-t/m}
 + B\sqrt{\frac{|\Xcal|\IGn}{2n}} \\
&\leq
\EE[\penl(\thetatrue, \uodatam)] + 
(1 - \EE[\penl(\thetatrue, \uodatam)])
\frac{(1+m)e^{-1/m}}{n}
 + B\sqrt{\frac{|\Xcal|\IGn}{2n}}.
\end{align*}
Here, the first step uses Theorem~\ref{thm:finiteactions}, 
the second step uses Lemma~\ref{lem:gouobound} for each $t$.
The third step rearranges terms and the last step 
bounds the sum by an integral to obtain,
\[
\sum_{t=1}^n e^{-t/m} \leq e^{-1/m} + \int_1^\infty e^{-t/m}\ud t
\leq e^{-1/m} + m  e^{-1/m}.
% \leq (1+m)e^{-1/m}
\]
Now, using $m=\gamma n$ and the fact that $1+x \leq e^x$ for $x>0$
we have 
$(1+m)e^{-1/m}/n = \gamma(1+\frac{1}{\gamma n})e^{-1/(\gamma n)}
\leq \gamma$.
The claim follows from the fact 
$\mu = 1-\penl$.
\end{proof}

\end{document}